\documentclass[11pt]{article}
\usepackage{amsthm}
\newtheorem{proposition}{Proposition}
\usepackage[preprint]{acl}

\usepackage{times}
\usepackage{latexsym}

\usepackage[T1]{fontenc}

\usepackage[utf8]{inputenc}

\usepackage{microtype}

\usepackage{inconsolata}

\usepackage{graphicx}
\usepackage{amsmath}
\usepackage{cleveref}

\usepackage{svg}
\usepackage{hyperref}       
\usepackage{url}            
\usepackage{booktabs}       
\usepackage{amsfonts}       
\usepackage{nicefrac}       
\usepackage{microtype}      
\usepackage{xcolor}         
 \usepackage{pifont}
\newcommand{\cmark}{\ding{51}}
\newcommand{\xmark}{\ding{55}}
\usepackage[table]{xcolor}

\title{Caching for the Future: Scrub Jay Episodic Memory Principles for Agent Memory Systems}

\author{
\begin{tabular}{cc}
\textbf{Kartikey Singh Bhandari} &
\textbf{Aarya Wadhwani} \\[0.25em]
\textbf{Dhruv Kumar} &
\textbf{Pratik Narang}
\end{tabular}
\\[0.75em]
Birla Institute of Technology and Science, Pilani
\\[0.25em]
\texttt{\{p20241006,f20230759\}@pilani.bits-pilani.ac.in}
\\
\texttt{\{dhruv.kumar,pratik.narang\}@pilani.bits-pilani.ac.in}
}

\begin{document}
\maketitle
\begin{abstract}
LLM agents that persist across sessions accumulate stored memories
whose validity varies enormously by content type, yet existing memory
architectures treat all memories as equally persistent and
systematically contaminate retrieved context with outdated facts.
We show that per-memory, type-conditioned temporal decay, a property
of western scrub jay episodic memory, can be operationalized as an
auto-classified coefficient $\pi_i$ in an external LLM-agent memory
store, yielding \textsc{ScrubJay-MEM}: each memory is encoded as a
jointly-bound What--Where--When tuple with an estimated perishability
$\pi_i$ and utility horizon $\tau_i$, retrieved by query-adaptive
scoring, and revised retroactively at $O(1)$ LLM calls per update.
We introduce the \emph{Temporal Generalization Test} (TGT), a benchmark
with held-out retention intervals and a Generalization Gap (GenGap)
metric. On TGT, ScrubJay-MEM is the only retrieval-based system with
substantially positive GenGap ($+0.108$); on MemoryAgentBench
EventQA-64k it improves F1 by $+2.66$ over Mem0 and $+3.09$ over
Qwen3-Embedding-4B under a llm backbone. A decay
ablation collapses GenGap by $5.7\times$, establishing type-conditioned
decay as necessary for the result. Gains narrow under stronger
backbones and reverse on fact-consolidation tasks, scoping the
contribution to temporal reasoning over perishable facts.
\end{abstract}

\section{Introduction}
\label{sec:intro}

A western scrub jay caching food in autumn does something
current LLM agent memory systems cannot.  Faced with a choice
between stored worms and stored peanuts, the bird recovers
worms when caches are fresh and peanuts when the worms have
rotted~\citep{clayton1998episodic,worsfold2025revisiting}.  It tracks not only what
it cached and where, but also \emph{when}, and uses this to
infer perishability.  The integrated What--Where--When (WWW)
trace it forms is robust to time in a way that LLM agent
memory architectures, despite recent
sophistication~\citep{packer2023memgpt, xu2025amem,
chhikara2025mem0}, have not yet achieved.  Recent reviews of
this literature argue the computational implications of WWW
memory remain
underexplored~\citep{salwiczek2010revisiting}.

This matters because LLM agents now operate over horizons
measured in weeks.  A user's profession may remain valid for
years; their meeting room for the morning is stale by
afternoon; their current branch name may not survive lunch.
As the memory store grows, the share of retrievals that
surface outdated facts grows with
it~\citep{wu2025longmemeval}.  The right memory at the wrong
time is no better than the wrong memory.

The naive fix (discount all memories by age) does not work.
Different memory types decay at fundamentally different rates.
A uniform decay either discounts stable knowledge too
aggressively or retains ephemeral facts too long.  A binary
short-term / long-term partition covers only two points on a
continuum that contains at least four distinguishable rates
(\S\ref{sec:utility}).  Worse, perishability is not static: a
memory initially classified as stable can become
time-sensitive when context shifts.  Solving the temporal
memory problem therefore requires \emph{per-memory},
\emph{type-conditioned}, and \emph{retroactively revisable}
decay: properties no agent memory system currently combines.

The scrub jay's brain provides a template.  We propose
\textsc{ScrubJay-MEM}, an agent memory architecture in which
each memory is stored as an Episodic Memory Unit: a
jointly-bound tuple of semantic content, task context,
timestamp, and an auto-classified perishability coefficient
$\pi_i$.  Retrieval combines all four signals through
query-adaptive weights $[\alpha, \beta, \gamma, \delta]$.  A
Prospective Memory Buffer pre-loads anticipated memories
before task execution, analogous to jays' future-oriented
caching~\citep{raby2007planning}.  Retroactive Contextual
Integration revises decay parameters when new information
arrives, mirroring how jays update cache values when
ecological signals
change~\citep{clayton2001integrated}.  Each component traces
to a cognitive observation rather than an engineering choice.

We evaluate ScrubJay-MEM on two complementary benchmarks.  On
MemoryAgentBench EventQA-64k~\citep{hu2025memoryagentbench},
the system achieves 61.58 F1, the strongest of the systems
evaluated, including A-MEM, Mem0, and Contriever.  We
additionally release the \emph{Temporal Generalization Test}
(TGT), a controlled benchmark with held-out retention
intervals, designed as the computational analog of Clayton and
Dickinson's retention-interval experiment.  On TGT,
ScrubJay-MEM is the only retrieval-based system with
substantially positive Generalization Gap ($+0.108$ vs.\
$\leq -0.022$ for all flat-retrieval baselines).  Ablating
type-conditioned decay collapses this gain by $5.7\times$,
establishing decay as \emph{necessary} for the GenGap result.

\paragraph{Contributions.}
\begin{itemize}

  \item \textbf{A biologically grounded memory architecture}
    (\S\ref{sec:method}) translating the scrub jay's WWW system
    into computational form, with four mechanisms (EMU,
    type-conditioned perishability decay, Retroactive
    Contextual Integration, Prospective Memory Buffer) derived
    from a single cognitive model.

  \item \textbf{Two formal results}
    (Proposition~\ref{prop:rci}, Proposition~\ref{prop:pmb};
    proofs in Appendix~\ref{app:proofs}): per-update
    boundedness with \textsc{what}-contraction for RCI, and
    expected sub-linear retrieval cost for the Prospective
    Memory Buffer.

  \item \textbf{The Temporal Generalization Test (TGT)}
    (\S\ref{sec:e3-tgt}; construction details in
    Appendix~\ref{app:tgt}), a controlled benchmark with
    held-out retention intervals, designed to measure whether memory systems
    generalise temporal decay knowledge to unseen intervals.

  \item \textbf{Type-conditioned decay as the necessary
    mechanism for TGT generalisation}
    (\S\ref{sec:experiments}). Ablating perishability decay
    collapses GenGap by $5.7\times$ while leaving
    staleness-detection accuracy intact, evidence that the
    GenGap gain requires the decay component, even if
    independent contributions from the remaining three
    primitives are not isolated here.

\end{itemize}

\begin{figure*}[t]
  \centering
  \includegraphics[width=\textwidth]{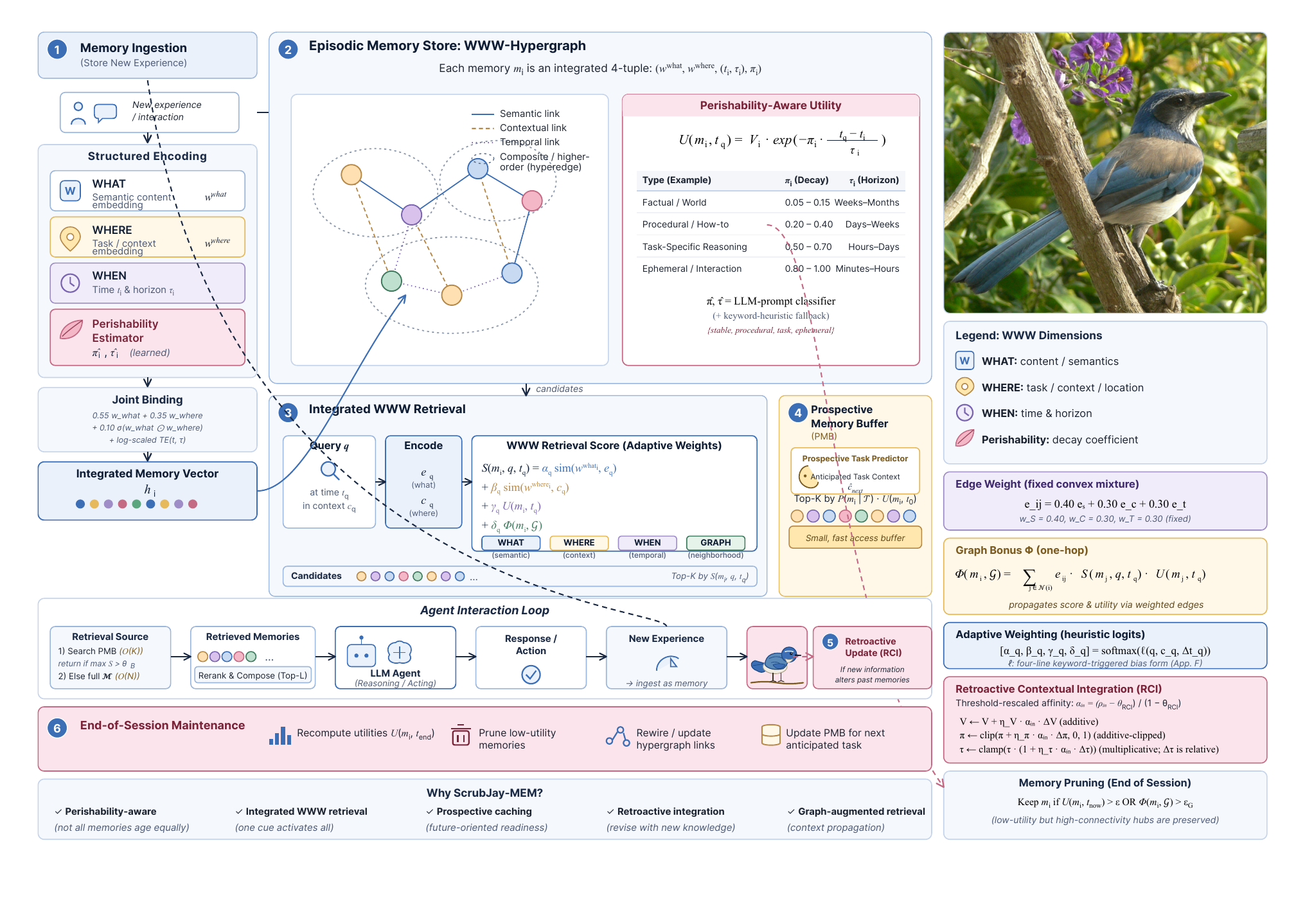}
  \caption{ScrubJay-MEM architecture (blueprint view).\protect\footnotemark{} The system
    encodes each experience as an integrated WWW (what--where--when)
    episodic memory with perishability $\pi$ and horizon $\tau$,
    organizes memories in a hypergraph, retrieves via a four-factor
    adaptive score, preloads likely-to-be-needed memories
    proactively (PMB), and retroactively updates utility and decay
    when new information arrives (RCI). \emph{The figure depicts the
    full architectural blueprint with trainable primitives; the implementation evaluated in this paper substitutes
    parameter-free surrogates at three points
    (\S\ref{sec:emu}--\ref{sec:retrieval}): a fixed-weight gated
    fusion in place of cross-attention binding, an LLM-prompt
    classifier with keyword fallback in place of the learned
    perishability MLP, and heuristic keyword-triggered logits in
    place of the learned adaptive-weighting head.}}

  \label{fig:architecture}
\end{figure*}

\section{Related Work}
\label{sec:related}

\paragraph{LLM agent memory architectures.}
External memory dominates long-horizon agent design.
MemGPT~\citep{packer2023memgpt} pages between main and
archival stores; Mem0~\citep{chhikara2025mem0} separates
short- and long-term tiers;
A-MEM~\citep{xu2025amem} introduces Zettelkasten linking
with LLM-driven evolution.
LiCoMemory~\citep{huang2025licomemory},
Zep~\citep{rasmussen2025zep},
Synapse~\citep{sun2025synapse},
ReadAgent~\citep{lee2024readagent},
and MemoryBank~\citep{zhong2024memorybank} add graph linking, temporal
entities, dynamic relations, gist compression, and
consolidation loops respectively.
Generative Agents~\citep{park2023generative} score by
importance, recency, and relevance;
Hindsight~\citep{zhao2025hindsight} reinterprets memories
textually post-hoc.
MemoryAgentBench~\citep{hu2025memoryagentbench} formalises
memory evaluation across four competencies; we adopt its
EventQA-64k as our primary benchmark.
\emph{ScrubJay-MEM departs from this landscape by assigning
each memory an auto-classified perishability coefficient
$\pi_i$ that governs type-conditioned exponential decay, and
by routing retrieval through query-adaptive
What--Where--When weights rather than fixed similarity.
Temporal validity becomes an inspectable architectural
parameter rather than an LLM-derived textual attribute.}

\paragraph{Temporal decay and forgetting.}
The Ebbinghaus
curve~\citep{ebbinghaus1885memory, anderson2004integrated,
brown2007temporal} models decay as a global function of time.
The Forgetting Transformer~\citep{liu2024forgetting} adds
forget gates inside attention;
FadeMem~\citep{liu2025fademem} and
MemoryBank~\citep{zhong2024memorybank} apply uniform
Ebbinghaus decay to external memory;
Oblivion~\citep{chang2025oblivion} and selective-forgetting
variants~\citep{huang2023selective, wang2023voyager,
mu2024learning, shi2023adaptive} make binary retain/discard
decisions.
\emph{Our work differs along three axes: decay is
\textbf{per-memory} rather than global, the rate is
\textbf{type-conditioned} through an auto-classified $\pi_i$,
and the rate is \textbf{retroactively revisable} through
parameter-space updates that cost $O(1)$ LLM calls regardless
of how many memories are affected.}

\footnotetext{Scrub jay photograph: Linda Tanner, "Western Scrub Jay," via Wikimedia Commons, licensed under CC BY 2.0. Source: \url{https://en.wikipedia.org/wiki/File:Aphelocoma_californica_3.jpg}.}
\paragraph{Episodic memory and anticipation.}
Episodic-memory architectures~\citep{tulving1972episodic,
pritzel2017neural, ritter2021rapid, wayne2018unsupervised}
predate LLM agents; WWW memory itself has appeared in
robotics and virtual
agents~\citep{duff2006www, brom2010episodic, crystal2009elements}
but not as an LLM-agent memory substrate.  Recent behavioural
work extends the corvid WWW literature with new controlled
analyses~\citep{worsfold2025revisiting}.
Prospective cognition~\citep{szpunar2010taxonomy,
schacter2012future} and
prefetching~\citep{smith1982cache, leviathan2023fast,
gao2023retrieval} provide antecedents for our Prospective
Memory Buffer. \emph{To our knowledge, ScrubJay-MEM is the first system to operationalise integrated WWW episodic memory specifically
for LLM agents, jointly binding content, context, and temporal metadata into a single retrieval-ready embedding. The Prospective Memory Buffer further translates anticipatory caching into an expected sub- linear retrieval mechanism with bounded per-update guarantees (Proposition~\ref{prop:rci}, Appendix~\ref{app:proofs}).}

\section{ScrubJay-MEM}
\label{sec:method}

We describe ScrubJay-MEM in five parts: the core memory
representation (\S\ref{sec:emu}--\ref{sec:utility}), the retrieval
mechanism (\S\ref{sec:retrieval}--\ref{sec:graph}), the two active
maintenance processes (\S\ref{sec:rci}--\ref{sec:pmb}), the encoding
strategy (\S\ref{sec:encoding}), and the pruning rule
(\S\ref{sec:pruning}).  Figure~\ref{fig:architecture} gives an overview.

\subsection{Episodic Memory Unit (EMU)}
\label{sec:emu}

Each memory $m_i$ is an integrated 4-tuple analogous to a single
caching episode:
\begin{equation}
  m_i = \bigl(\,
    \mathbf{w}_i^{\textsc{what}},\;
    \mathbf{w}_i^{\textsc{where}},\;
    (t_i,\,\tau_i),\;
    \pi_i
  \,\bigr),
  \label{eq:emu}
\end{equation}
where $\mathbf{w}_i^{\textsc{what}}\!\in\!\mathbb{R}^{d_1}$ encodes
the \emph{semantic content} of the memory (food type $\to$ information
content), $\mathbf{w}_i^{\textsc{where}}\!\in\!\mathbb{R}^{d_2}$
encodes the \emph{task context} in which it was created (cache site
$\to$ conversation/task state), $t_i$ is the creation timestamp,
$\tau_i\!\in\!\mathbb{R}_{+}$ is the estimated utility horizon, and
$\pi_i\!\in\!(0,1]$ is the \emph{perishability coefficient}
($\pi\!\approx\!1$ for ephemeral facts, $\pi\!\approx\!0$ for stable
knowledge).

The four fields are combined into a single binding vector
$\mathbf{h}_i$ via a deterministic gated fusion that approximates
integrated WWW retrieval at parameter-free cost. Let
$\mathbf{g}_i = \sigma\bigl(\mathbf{w}_i^{\textsc{what}} \odot
\mathbf{w}_i^{\textsc{where}}\bigr)$ be the elementwise sigmoid of the
Hadamard product (a content-context coactivation gate). The binding is
a fixed convex mixture
\begin{equation}
  \mathbf{h}_i
  \;=\; 0.55\,\mathbf{w}_i^{\textsc{what}}
        + 0.35\,\mathbf{w}_i^{\textsc{where}}
        + 0.10\,\mathbf{g}_i,
  \label{eq:binding}
\end{equation}
followed by an L2 normalisation. Log-scaled temporal positional
offsets $\log(t_i + 1)$ and $\log(\tau_i + 1)$ (each modulo a fixed
period) are added to the first two coordinates of $\mathbf{h}_i$ to
encode \textsc{when}. The fusion weights and gate are fixed
(parameter-free); replacing them with trained cross-attention is left
to future work.

\subsection{Perishability-Aware Utility}
\label{sec:utility}

The temporal utility of $m_i$ at query time $t_q$ is
\begin{equation}
  U(m_i, t_q)
  = V_i \cdot
    \exp\!\Bigl(-\pi_i \cdot \frac{t_q - t_i}{\tau_i}\Bigr),
  \label{eq:utility}
\end{equation}
with base value $V_i\!\in\!\mathbb{R}_{+}$ (learned).  Critically,
the decay rate is \emph{type-conditioned}: stable world knowledge
($\pi\!\approx\!0.05$) retains utility over long horizons while
ephemeral interaction details ($\pi\!\approx\!0.9$) decay rapidly
(Table~\ref{tab:perishability}).

Perishability and horizon are estimated at storage time by an
LLM-prompt classifier with a deterministic keyword-heuristic
fallback. The LLM is asked to assign one of four labels
$\{\text{factual},\,\text{procedural},\,\text{task\_specific},\,
\text{ephemeral}\}$ and to emit a numerical $\pi \in (0,1]$ and
$\tau \in \mathbb{R}_{+}$ as strict JSON; on parse failure or LLM
unavailability the system falls back to a small keyword rule over the
memory text and context (e.g., \texttt{today}/\texttt{immediate}
$\to$ ephemeral with $\pi\!=\!0.9$;
\texttt{how to}/\texttt{steps}/\texttt{procedure} $\to$ procedural
with $\pi\!=\!0.3$). Estimates are clipped to
$[\tau_{\min}, \tau_{\max}]$ to bound the utility horizon. We refer
to this classifier compactly as
\begin{equation}
  \hat{\pi}_i,\;\hat{\tau}_i
  = \phi\!\bigl(
      \text{text}_i,\,\text{ctx}_i,\,
      \mathbf{w}_i^{\textsc{what}},\,
      \mathbf{w}_i^{\textsc{where}}
    \bigr),
  \label{eq:pi-est}
\end{equation}
emphasising that the production of $(\hat{\pi}, \hat{\tau})$ in this
v1 system is an LLM call plus a heuristic fallback rather than a
trained neural network; the classifier prompt is reproduced verbatim
in Appendix~\ref{app:hyperparams}. A trained-MLP variant is left to
future work.

\subsection{Integrated WWW Retrieval}
\label{sec:retrieval}

Given query $\mathbf{q}$ at time $t_q$ with task context
$\mathbf{c}_q$, the retrieval score for memory $m_i$ is
\begin{equation}
\small
\begin{aligned}
S(m_i, \mathbf{q}, t_q)
&= \alpha_q\,
  \underbrace{
    \operatorname{sim}\!\left(
      \mathbf{w}_i^{\textsc{what}}, \mathbf{e}_q
    \right)
  }_{\textbf{What}}
 + \beta_q\,
  \underbrace{
    \operatorname{sim}\!\left(
      \mathbf{w}_i^{\textsc{where}}, \mathbf{c}_q
    \right)
  }_{\textbf{Where}}                                      \\[2pt]
&\quad
 + \gamma_q\,
  \underbrace{
    U(m_i, t_q)
  }_{\textbf{When}}
 + \delta_q\,
  \underbrace{
    \Phi(m_i, \mathcal{G})
  }_{\textbf{Graph}} .
\end{aligned}
\label{eq:score}
\end{equation}
where $\mathbf{e}_q = \operatorname{Enc}(\mathbf{q})$. The weights
$[\alpha_q,\beta_q,\gamma_q,\delta_q]$ are \emph{query-adaptive},
produced by a heuristic scoring function that combines fixed base
logits, embedding-norm scales, and keyword-triggered query biases:
\begin{equation}
\small
\begin{aligned}
  \mathbf{w}_q
  &=
  \begin{bmatrix}
    \alpha_q & \beta_q & \gamma_q & \delta_q
  \end{bmatrix}^{\!\top}                                      \\[-1pt]
  &=
  \operatorname{softmax}\!
  \begin{bmatrix}
    1.5 + b_q^{\textsc{what}}  + 0.20\,\|\mathbf{e}_q\| \\
    1.0 + b_q^{\textsc{where}} + 0.10\,\|\mathbf{c}_q\| \\
    1.0 + b_q^{\textsc{when}}
      + \min\!\bigl(2.0,\tfrac{1}{5}\log(\Delta t_q + 1)\bigr) \\
    1.0 + b_q^{\textsc{graph}}
  \end{bmatrix}.
\end{aligned}
\label{eq:adaptive-weights}
\end{equation}
where the bias terms $b_q^X$ are keyword-triggered scalars that
detect query intent
(e.g.\ \textit{latest}/\textit{recent}/\textit{today}/\textit{stale}
contribute $+1.5$ to $b_q^{\textsc{when}}$;
\textit{context}/\textit{project}/\textit{session}/\textit{workspace}
contribute $+1.0$ to $b_q^{\textsc{where}}$;
\textit{what}/\textit{which}/\textit{fact}/\textit{detail}/\textit{remember}
contribute $+0.8$ to $b_q^{\textsc{what}}$;
\textit{related}/\textit{connected}/\textit{neighbor}/\textit{link}
contribute $+0.8$ to $b_q^{\textsc{graph}}$). Temporally urgent
queries thus up-weight $\gamma$ (When) and context-sensitive queries
up-weight $\beta$ (Where). The weighting is parameter-free in this v1
system; a learned weighting head $W_\theta$ that replaces the
heuristic is left to future work. This contrasts with A-MEM's
fixed-weight keyword/embedding lookup.

\subsection{Episodic Hypergraph}
\label{sec:graph}

We maintain a typed graph $\mathcal{G}=(\mathcal{V},\mathcal{E})$
over the memory store with three edge classes: semantic
(\textsc{what--what}), contextual (\textsc{where--where}), and
temporal co-occurrence (\textsc{when--when}), combined via a fixed
convex mixture (in our experiments $w_S\!=\!0.40$,
$w_C\!=\!0.30$, $w_T\!=\!0.30$; see
Appendix~\ref{app:hyperparams}):
\begin{equation}
  e_{ij}
  = w_S\,e_{ij}^S + w_C\,e_{ij}^C + w_T\,e_{ij}^T,
  \label{eq:edge}
\end{equation}
where $e_{ij}^S$ and $e_{ij}^C$ are thresholded cosine similarities
between the respective embeddings, and
$e_{ij}^T = \exp(-\lambda_T |t_i - t_j|)\,\mathbb{1}[|t_i - t_j| <
\Delta_T]$.  The graph bonus for $m_i$ is computed by one-hop message
passing:
\begin{equation}
  \Phi(m_i, \mathcal{G})
  = \sum_{j \in \mathcal{N}(i)}
    e_{ij}\,S(m_j,\mathbf{q},t_q)\,U(m_j,t_q),
  \label{eq:graph-bonus}
\end{equation}
which propagates both semantic relevance \emph{and} temporal utility
from neighbors: links to decayed memories are automatically
down-weighted.  This extends A-MEM's semantic-only Zettelkasten graph.

\subsection{Retroactive Contextual Integration (RCI)}
\label{sec:rci}

When new information $\mathbf{n}$ arrives at time $t_n$ (e.g., a user
preference change), RCI updates the \emph{latent parameters} of all
affected memories without a per-memory LLM call.

\paragraph{Step 1: Identify affected memories.}
Compute the cosine relevance
$\rho_{in} = \operatorname{sim}(\mathbf{w}_i^{\textsc{what}},
\mathbf{w}_n^{\textsc{what}})$ and gate on the threshold
$\theta_{\text{RCI}}$. Memories with $\rho_{in} \le \theta_{\text{RCI}}$
are skipped. For memories above the threshold, we rescale the
similarity into a threshold-relative affinity
$\alpha_{in} = (\rho_{in} - \theta_{\text{RCI}})/(1 - \theta_{\text{RCI}})
\in (0, 1]$ that drives the magnitude of all subsequent updates.

\paragraph{Step 2: Update parameters.}
A \emph{single} LLM call parses $\mathbf{n}$ into bounded update
deltas $(\Delta V(\mathbf{n}),\,\Delta\pi(\mathbf{n}),\,
\Delta\tau(\mathbf{n}))$ and a textual summary;
$\Delta\tau$ is interpreted as a relative ratio (so
$\Delta\tau = +0.2$ denotes a $+20\%$ change in $\tau$). The deltas
are applied to each affected memory via lightweight vector
operations, modulated by the affinity $\alpha_{in}$ from Step~1:
\begin{align}
  V_i
  &\leftarrow
  \max\!\bigl(
    0,\,
    V_i + \eta_V\,\alpha_{in}\,\Delta V(\mathbf{n})
  \bigr),
  \label{eq:rci-v} \\[2pt]
  \pi_i
  &\leftarrow
  \operatorname{clip}\!\bigl(
    \pi_i + \eta_\pi\,\alpha_{in}\,\Delta\pi(\mathbf{n}),
    \;0,\;1
  \bigr),
  \label{eq:rci-pi} \\[2pt]
  \tau_i
  &\leftarrow
  \operatorname{clip}\!\Bigl(
    \tau_i \cdot
    \max\!\bigl(
      0.1,\,
      1 + \eta_\tau\,\alpha_{in}\,\Delta\tau(\mathbf{n})
    \bigr),
    \notag \\
  &\hspace{3.3cm}
    \tau_{\min},\,\tau_{\max}
  \Bigr).
  \label{eq:rci-tau}
\end{align}
The $V$ and $\pi$ updates are additive offsets (with $\pi$ clipped to
$[0,1]$); the $\tau$ update is multiplicative because $\Delta\tau$ is
a relative ratio. All three are bounded per update by the LLM's
fixed delta dynamic range and the affinity gate
(Appendix~\ref{app:proofs}); the clip on $\tau$ keeps the horizon
inside the operating interval $[\tau_{\min}, \tau_{\max}]$
regardless of repeated updates. Learning rates
$(\eta_V,\,\eta_\pi,\,\eta_\tau)$ are reported in
Appendix~\ref{app:hyperparams}.

\paragraph{Step 3: Soft semantic fusion.}
The \textsc{what} embedding is moved toward the new evidence by a
damped step using the same affinity gate, followed by L2
re-normalisation:
\begin{equation}
\small
\begin{aligned}
  \mathbf{w}_i^{\textsc{what}}
  &\leftarrow
  \mathrm{L2}\!\Bigl(
    \mathbf{w}_i^{\textsc{what}}
    + \eta_w\,\alpha_{in}
    \bigl(
      \mathbf{w}_n^{\textsc{what}}
      - \mathbf{w}_i^{\textsc{what}}
    \bigr)
  \Bigr).
\end{aligned}
\label{eq:rci-embed}
\end{equation}
Total per-event cost is $O(1)$ LLM calls (for delta parsing) plus
$O(|\{i:\rho_{in}>\theta_{\text{RCI}}\}|)$ vector operations,
compared with A-MEM's $O(N)$ LLM calls for $N$ affected memories.

\begin{proposition}[RCI Per-Update Boundedness and \textsc{what}-Contraction]
\label{prop:rci}
Let the parsed deltas be bounded:
$|\Delta V(\mathbf{n})|\!\le\!\Delta V_{\max}$,
$|\Delta\pi(\mathbf{n})|\!\le\!\Delta\pi_{\max}$, and
$\Delta\tau(\mathbf{n})\!\in\![\Delta\tau_{\min},\Delta\tau_{\max}]$
with $1 + \eta_\tau\,\Delta\tau_{\min} \ge 0.1$ (enforced by the clamp
in Eq.~\ref{eq:rci-tau}). Then for any memory $m_i$ and any affinity
$\alpha_{in}\!\in\![0,1]$:
\textbf{(i)} the $V$-update satisfies
$|\Delta V_i|\!\le\!\eta_V\,\Delta V_{\max}$;
\textbf{(ii)} the $\pi$-update satisfies
$|\Delta\pi_i|\!\le\!\eta_\pi\,\Delta\pi_{\max}$ inside the $[0,1]$
clip;
\textbf{(iii)} $\tau_i$ remains in $[\tau_{\min},\tau_{\max}]$ after
clipping;
\textbf{(iv)} the \textsc{what}-update~\eqref{eq:rci-embed} is, prior
to renormalisation, a contraction toward $\mathbf{w}_n^{\textsc{what}}$
with Lipschitz constant $(1-\eta_w\,\alpha_{in})$; whenever
$\eta_w\,\alpha_{in}\!<\!1$ and $\mathbf{w}_n^{\textsc{what}}$ is
held fixed, repeated application converges geometrically.
\end{proposition}
\noindent\textit{Proof.}  See Appendix~\ref{app:proofs}.

\subsection{Prospective Memory Buffer (PMB)}
\label{sec:pmb}

Before executing task $\mathcal{T}$, ScrubJay-MEM pre-loads a buffer
of $K$ memories most likely to be needed, inspired by the jay's
anticipatory caching behavior.

\paragraph{Anticipatory scoring.}
\begin{equation}
  P(m_i \mid \mathcal{T})
  = \sigma\!\bigl(
    \mathbf{v}_\mathcal{T}^\top \mathbf{w}_i^{\textsc{what}}
    + \mathbf{u}_\mathcal{T}^\top \mathbf{w}_i^{\textsc{where}}
    + b_\mathcal{T}
  \bigr),
  \label{eq:pmb-score}
\end{equation}
where $(\mathbf{v}_\mathcal{T}, \mathbf{u}_\mathcal{T}, b_\mathcal{T})$
are light-weight task-encoder parameters.

\paragraph{Buffer population.}
$\mathcal{B}_\mathcal{T}
  = \operatorname{top\text{-}K}\{P(m_i\mid\mathcal{T})\cdot U(m_i,t_0)\}_{i=1}^{|\mathcal{M}|}$.

\paragraph{Two-stage retrieval.}
At query time the system first searches the buffer ($O(K)$); if the
best score exceeds a confidence threshold $\theta_\mathcal{B}$, it is
returned immediately.  Otherwise, a fallback full-store search ($O(N)$)
is performed.

\begin{proposition}[Sub-linear Retrieval]
\label{prop:pmb}
Let $P_{\mathrm{hit}} = \Pr[\text{answer}\in\mathcal{B}_\mathcal{T}]$.
Then expected retrieval cost is
$K + (1 - P_{\mathrm{hit}})\cdot N$.
As the task encoder improves, $P_{\mathrm{hit}}\!\to\!1$ and
retrieval becomes $O(K)$.
\end{proposition}

\subsection{Future-Oriented Encoding}
\label{sec:encoding}

Following evidence that prospective encoding improves later
recall~\citep{szpunar2010taxonomy}, each new experience $x_i$ is
augmented at storage time with a brief prospective annotation:
\begin{equation}
\small
\begin{aligned}
  \tilde{x}_i
  &=
  \bigl[
    x_i\,;\;
    \operatorname{LLM}_\theta(
      \text{``When useful?''}\,\|\,x_i
    )
  \bigr],                                                   \\
  \mathbf{w}_i^{\textsc{what}}
  &=
  \operatorname{Enc}(\tilde{x}_i).
\end{aligned}
\label{eq:prosp-enc}
\end{equation}
This biases the \textsc{what} embedding toward anticipated retrieval
scenarios, improving recall on future queries.

\subsection{Value-Weighted Memory Consolidation}
\label{sec:pruning}

At session boundaries, memories whose utility has decayed below
threshold $\epsilon$ are pruned, unless they serve as high-connectivity
hubs in $\mathcal{G}$:
\begin{equation}
\small
\begin{aligned}
  \mathcal{M}'
  =
  \bigl\{
    m_i \in \mathcal{M}
    :\;&
    U(m_i, t_{\text{now}}) > \epsilon
    \\
    &\text{or}\;
    \Phi(m_i, \mathcal{G}) > \epsilon_G
  \bigr\}.
\end{aligned}
\label{eq:prune}
\end{equation}
The second condition prevents fragmentation of the episodic hypergraph
by retaining low-utility but structurally important nodes.

\subsection{Full Pipeline Summary}
\label{sec:pipeline}

The complete system operates in four phases.
\textbf{(1)~Storage}: new experiences are prospectively annotated
(\S\ref{sec:encoding}), encoded into an EMU (\S\ref{sec:emu}), assigned
auto-estimated $(\hat\pi,\hat\tau)$ (\S\ref{sec:utility}), and
integrated into $\mathcal{G}$ (\S\ref{sec:graph}).
\textbf{(2)~Prospective loading}: the PMB pre-populates a buffer of
$K$ anticipated memories for the upcoming task (\S\ref{sec:pmb}).
\textbf{(3)~Retrieval}: incoming queries are scored against the buffer
(then full store on miss) via adaptive WWW scoring
(\S\ref{sec:retrieval}).
\textbf{(4)~Maintenance}: RCI updates affected memory parameters when
new information arrives (\S\ref{sec:rci}); consolidation prunes
low-utility memories at session end (\S\ref{sec:pruning}).

\section{Experiments}
\label{sec:experiments}

\subsection{Temporal Reasoning on MemoryAgentBench (EventQA-64k)}
\label{sec:mab-eventqa}

We evaluate ScrubJay-MEM on the EventQA subset of MemoryAgentBench
\citep{hu2025memoryagentbench}, which tests an agent's ability to retrieve and reason over
time-ordered event sequences embedded in long narrative passages. EventQA
is the most temporally demanding subset of MAB, making it the closest
public analog to the kind of task ScrubJay-MEM's perishability
mechanism is designed for.

\paragraph{Setup.}
The benchmark consists of 5 long-form contexts (median length 64K
tokens) with 100 questions each, for a total of 500 queries. We use
MAB's official scorer (\texttt{metrics\_summarization}) which reports
exact-match (EM), token-level F1, substring EM (sEM), and ROUGE-L
F1, with no LLM-as-judge component. All systems share the same
backbone (\texttt{llama3.1:8b}) and embedding model
(\texttt{nomic-embed-text}, 768-dim), served locally via Ollama for
reproducibility and to eliminate provider-specific confounds.
Five baselines are evaluated: lexical (BM25), dense
(Contriever \citep{izacard2022contriever}, Qwen3-Embedding-4B \citep{qwen2025embedding}), and agentic-memory
(Mem0 \citep{chhikara2025mem0}, A-MEM \citep{xu2025amem}). Numbers are reported with the same
top-$k=5$ retrieval depth and chunk size (4096) across systems.

\paragraph{Results.}
ScrubJay-MEM achieves the highest F1 (\textbf{61.58}) and the highest
exact-match (\textbf{41.00}) on EventQA-64k, outperforming the strongest
agentic baseline (Mem0) by +2.66 F1 / +5.00 EM and the strongest
embedding baseline (Qwen3-Embedding-4B) by +3.09 F1 / +4.40 EM
(\Cref{tab:eventqa-llama8b}).

\begin{table}[t]
  \centering
  \small
  \caption{EventQA-64k results with the \texttt{llama3.1:8b} backbone.
  All systems use MAB's official scorer. \textbf{Bold}: best in column.}
  \label{tab:eventqa-llama8b}
  \resizebox{\columnwidth}{!}{%
  \begin{tabular}{lcccc}
    \toprule
    System & EM & F1 & sEM & R-L \\
    \midrule
    BM25 & 33.60 & 55.18 & 36.80 & 56.74 \\
    A-MEM \citep{xu2025amem} & 37.00 & 56.46 & 41.20 & 58.20 \\
    Contriever \citep{izacard2022contriever} & 38.60 & 57.67 & 40.60 & 59.42 \\
    Qwen3-Emb.-4B \citep{qwen2025embedding} & 36.60 & 58.49 & 39.00 & 60.15 \\
    Mem0 \citep{chhikara2025mem0} & 36.00 & 58.92 & 37.40 & 60.84 \\
    \midrule
    \textbf{ScrubJay-MEM} &
    \textbf{41.00} & \textbf{61.58} & \textbf{42.80} & \textbf{63.39} \\
    \bottomrule
  \end{tabular}%
  }
\end{table}\paragraph{Sensitivity to backbone capability.}
A natural concern is whether ScrubJay-MEM's advantage persists with a
stronger generation model. To probe this, we re-ran the full sweep with
\texttt{qwen3:30b-instruct} (\Cref{tab:eventqa-qwen30b}, appendix).
With the stronger backbone, dense-retrieval baselines close the gap
considerably (Contriever F1 78.68 vs.\ ScrubJay-MEM 72.82), suggesting
that some of the temporal benefit our architecture confers can be
matched by a sufficiently capable LLM reasoning over dense retrievals.
We frame this honestly: \emph{ScrubJay-MEM's contribution is largest
in the resource-constrained regime that characterizes most deployed
memory agents}, where the backbone LLM is not large enough to
internally compensate for a temporally-naive retriever.

\paragraph{Honest scoping: where ScrubJay-MEM does \emph{not} help.}
We also evaluated on MAB's Conflict-Resolution subsets
(Factconsolidation-MH, Factconsolidation-SH). On these tasks, where
factual updates supersede earlier statements but stale facts must
remain visible for the consolidation step, our perishability mechanism
is actively miscalibrated: type-conditioned decay suppresses the very
memories the task requires. We report this in the appendix
(\Cref{tab:mab-cr-appendix}) rather than the main results, but flag it
explicitly: ScrubJay-MEM targets \emph{temporal reasoning over
perishable facts}, not generic fact retrieval. This delineation is a
feature, not a limitation: the next subsection (\Cref{sec:e3-tgt})
quantifies the architectural property responsible for our win.

\subsection{Temporal Generalization to Unseen Retention Intervals}
\label{sec:e3-tgt}

EventQA establishes that ScrubJay-MEM helps on temporal benchmarks
overall. This subsection isolates \emph{why}: we introduce the
Temporal Generalization Test (TGT), a controlled diagnostic that
asks whether a memory system can generalize its temporal behavior to
retention intervals it has never been calibrated on. The benchmark is
designed as the computational analog of the scrub jay retention
interval experiment of \citet{clayton2003can}, which demonstrated that western
scrub jays generalize cache-decay knowledge to intermediate, unseen
retention intervals, a property we argue any biologically-plausible
memory system should exhibit.

\paragraph{Benchmark design.}
TGT consists of 20 instances, each comprising 96 memories sampled
across four perishability classes (Ephemeral, Session-Specific,
Durable Preference, Stable Knowledge; 24 per class) and 66 queries
issued at five retention intervals:
$I_1$ (immediate, 0--1 sessions),
$I_2$ (short, 3--5),
$I_3$ (medium, 8--12, \textbf{unseen}),
$I_4$ (long, 20--25), and
$I_5$ (very long, 40--60, \textbf{unseen}).
$I_3$ tests interpolation to an unseen interval between two seen
ones; $I_5$ tests extrapolation beyond the calibration range. The
full benchmark contains 1{,}320 queries.

\paragraph{Metric: Generalization Gap.}
We define the \emph{Generalization Gap} as
\begin{equation}
  \mathrm{GenGap} \;=\; \mathrm{Acc}(I_2) + \mathrm{Acc}(I_4) - 2\cdot\mathrm{Acc}(I_3),
  \label{eq:gengap}
\end{equation}
which measures how far the system's accuracy at the unseen
interpolation interval falls below the linear interpolation of its
accuracy at the two adjacent seen intervals. A system with continuous
temporal modeling should have $\mathrm{GenGap} \approx 0$; a system
that has memorized behavior only at the seen intervals will have
$\mathrm{GenGap} \ll 0$. Positive GenGap indicates \emph{robust}
generalization—accuracy at the unseen interval at or above the
seen-interval baseline. We additionally report the Temporal
Generalization Score $\mathrm{TGS} = \tfrac{1}{5}\sum_k
\mathrm{Acc}(I_k)$, and Combined / Factual / Staleness accuracy at the
per-query level (definitions in \Cref{app:tgt-metrics}).

\paragraph{Baselines.}
We compare against a stratified set of nine baselines spanning the
spectrum of memory designs: a random control (B0), an LLM-only
condition with no retrieval (B1), lexical retrieval (B2: BM25),
dense retrieval (B3: \texttt{nomic-embed-text} top-1), hybrid
retrieval (B4: BM25 + dense via reciprocal rank fusion), a recency-prior
baseline (B5: dense + global exponential recency), vanilla RAG (B6:
dense retrieval + LLM answer generation), and RAG with explicit
temporal prompting (B7: B6 with session-age annotations and
type-perishability instructions in the prompt). The B5 baseline is
included specifically to test whether a single global decay rate
matches type-conditioned perishability; B7 tests whether a strong LLM
with explicit age annotations can match an explicit decay model.

\paragraph{Setup.}
All TGT systems share \texttt{qwen3:30b-instruct} as the generation
model (where applicable) and \texttt{nomic-embed-text} (768-dim) as
the embedder; retrieval depth is top-$k=10$ and we use a single seed
($42$). The benchmark itself is naturalised at construction time by
\texttt{qwen3:30b-instruct} (Appendix~\ref{app:tgt}); systems consume
the released benchmark unchanged. To avoid self-preference bias
\citep{panickssery2024llmselfpref}, we use \texttt{llama3.2:3b} as
the answer-correctness judge, a different model family (Meta) from
the Alibaba-trained generation backbone, chosen for cross-family
separation rather than capacity. Each query is scored on (i)~factual
correctness against ground-truth memory content and (ii)~temporal
validity (staleness) judgment; combined accuracy requires both.

\paragraph{Results.}
\Cref{tab:e3-main} shows the full E3 TGT results.

\begin{table*}[t]
  \centering
  \small
  \caption{Temporal Generalization Test results. 20 instances,
  1{,}320 queries, $\texttt{qwen3:30b-instruct}$ generation +
  $\texttt{llama3.2:3b}$ judge (cross-family).
  \textbf{Bold}: best in column among retrieval-based systems
  (excluding LLM-only and trivial controls).
  GenGap > 0 indicates robust generalization to the unseen
  interpolation interval.
  Tier categorises systems by per-query LLM usage:
  \emph{retrieval-based} systems (including ours) produce the answer
  as the top-1 retrieved memory text without invoking the generator
  at query time; \emph{RAG} systems retrieve then generate via the
  LLM. ScrubJay-MEM uses LLM calls at ingest and on RCI events
  (\S\ref{sec:utility},~\ref{sec:rci},~\ref{sec:encoding}) but
  not at query time.}
  \label{tab:e3-main}
  \begin{tabular}{llccccc}
    \toprule
    System & Tier & Combined $\uparrow$ & Factual $\uparrow$ & Staleness $\uparrow$ & TGS $\uparrow$ & GenGap $\uparrow$ \\
    \midrule
    \multicolumn{7}{l}{\emph{LLM-based and trivial controls}} \\
    \quad Random (B0)              & Trivial    & 28.6 & 29.3 & 84.1 & 0.296 & $+0.071$ \\
    \quad LLM-only (B1)            & LLM-only   & 44.4 & 44.4 & 95.2 & 0.459 & $+0.181$ \\
    \quad RAG-vanilla (B6)         & RAG        & 62.7 & 67.1 & 91.5 & 0.633 & $+0.039$ \\
    \quad RAG-temporal (B7)        & RAG        & 67.2 & 71.5 & 90.7 & 0.678 & $-0.018$ \\
    \midrule
    \multicolumn{7}{l}{\emph{Retrieval-based systems}} \\
    \quad BM25 (B2)                & Lexical    & 35.2 & 40.2 & 77.7 & 0.353 & $+0.039$ \\
    \quad Dense top-1 (B3)         & Dense      & 37.0 & 41.7 & 73.6 & 0.372 & $-0.022$ \\
    \quad BM25 $\!\oplus\!$ Dense RRF (B4) & Hybrid     & \textbf{39.4} & \textbf{45.0} & 74.6 & \textbf{0.395} & $-0.054$ \\
    \quad Dense + recency (B5)     & Temporal   & 37.5 & 41.9 & 74.5 & 0.377 & $-0.046$ \\
    \quad ScrubJay-MEM (no decay)  & Ours-abl.  & 35.8 & 39.1 & 82.3 & 0.366 & $+0.019$ \\
    \quad \textbf{ScrubJay-MEM (full)} & \textbf{Ours} & 37.1 & 39.8 & \textbf{82.3} & 0.378 & $\mathbf{+0.108}$ \\
    \bottomrule
  \end{tabular}
\end{table*}

We organize the discussion around three findings.

\textbf{Finding 1: ScrubJay-MEM is the only retrieval-based system
that generalizes to unseen intervals.}
Among all retrieval-based methods, only ScrubJay-MEM (full) achieves
substantially positive GenGap ($+0.108$). Every flat-retrieval
baseline (lexical, dense, hybrid, even hybrid with a recency
prior) has \emph{negative} GenGap, meaning their accuracy at the
unseen $I_3$ interval is below the linear interpolation of the seen
$I_2$ and $I_4$ accuracies. The next-best retrieval system on GenGap
(hybrid RRF) sits at $-0.054$, giving ScrubJay-MEM a
\textbf{16.2-percentage-point margin} on this metric.

\textbf{Finding 2: Type-conditioned decay, not flat recency, is the
mechanism.}
The decay ablation (\textsc{ScrubJay-MEM (no decay)}) collapses GenGap
from $+0.108$ to $+0.019$ (a $5.7\times$ reduction) while staleness
accuracy is preserved (both variants at $82.3\%$). The flat-recency
baseline (B5, GenGap $-0.046$) performs \emph{worse} than the
no-temporal-model baseline (B3, GenGap $-0.022$), directly refuting
the conjecture that a single global decay rate is sufficient. Naive
recency penalizes durable memories along with ephemeral ones;
type-conditioned perishability decays only what should decay
(\Cref{fig:e3-decay}).

\textbf{Finding 3: LLM-based methods achieve higher combined accuracy
but via a different mechanism.}
RAG-temporal (B7) achieves the highest combined accuracy on E3 ($67.2\%$
vs.\ our $37.1\%$), driven by the LLM's ability to reason over
session-age annotations at inference time. We do not contest this
result; we contextualize it. B7's GenGap is $-0.018$, indicating that
its temporal competence does not generalize as a learned property but
must be re-derived at each query through prompt-level reasoning, at
the cost of one LLM call per query. ScrubJay-MEM encodes temporal
validity as architectural parameters ($\pi, \tau$), enabling
sub-linear retrieval (via the Prospective Memory Buffer,
\Cref{sec:pmb}) and an inspectable decay curve. We view the two
approaches as complementary rather than competing: prompting captures
temporal reasoning at inference cost; perishability captures it at
architectural cost. The two regimes occupy distinct Pareto frontiers in
the joint (combined, staleness) plane (\Cref{fig:e3-pareto}):
ScrubJay-MEM sits at the top of the retrieval frontier, matching
LLM-grade staleness judgment without a per-query LLM call.

\paragraph{A note on the LLM-only baseline.}
\textsc{LLM-only} (B1) has the highest GenGap of all systems ($+0.181$)
but only $44.4\%$ factual accuracy. This is an artifact of low-accuracy
systems exhibiting no systematic per-interval bias: when answers are
near-random with respect to memory content, they are also near-random
with respect to retention interval, giving artificially positive
GenGap. Among retrieval-grounded systems with factual accuracy above
$35\%$, ScrubJay-MEM's GenGap is the highest by a substantial margin.


\section{Conclusion}
\label{sec:conclusion}

ScrubJay-MEM translates four mechanisms of western scrub jay
episodic memory---integrated WWW encoding, a per-memory perishability
coefficient $\pi_i$, retroactive contextual integration at $O(1)$
LLM calls per update, and a Prospective Memory Buffer---into an
inspectable architecture for LLM-agent memory. Under
\texttt{llama3.1:8b}, the system improves EventQA-64k F1 by $+2.66$
over Mem0 and $+3.09$ over Qwen3-Embedding-4B; on the Temporal
Generalization Test introduced here, it is the only retrieval-based
system with substantially positive GenGap ($+0.108$), and a
\textsc{no-decay} ablation collapses this gain by $5.7\times$,
isolating type-conditioned decay as the responsible mechanism and
operationalising a comparative-cognition primitive as an inspectable
architectural parameter of LLM-agent memory.

\section*{Limitations}
\label{sec:limitations}

\paragraph{Scope.}
Our gains concentrate on temporal reasoning over perishable facts.
Under a stronger backbone (\Cref{app:eventqa-qwen}), dense retrieval
matches us on EventQA-64k, and on MemoryAgentBench Conflict-Resolution
(\Cref{app:mab-cr}) flat retrieval is better-suited than
type-conditioned decay; fact-consolidation requires that stale facts
remain visible. ScrubJay-MEM is therefore best understood as a
temporally-perishable fact retriever, complementary to flat retrievers
for consolidation tasks and to prompt-level temporal reasoning for
systems with abundant backbone capacity.

\paragraph{TGT is a controlled diagnostic.}
TGT is LLM-naturalised from deterministic templates with held-out retention intervals and a known validity schedule, properties that enable a mechanistic claim about generalisation but limit external
validity to in-distribution conversational logs. TGT's four perishability classes share $\pi$-ranges with our internal taxonomy
(\Cref{tab:perishability,tab:tgt-types}); the no-decay ablation, which
collapses GenGap by $5.7\times$, controls for the concern that the
benchmark and the architecture share categorical primitives. Reported
numbers use \texttt{qwen3:30b-instruct} for generation and a
cross-family \texttt{llama3.2:3b} judge to mitigate self-preference
bias~\citep{panickssery2024llmselfpref}; the qualitative ranking we
report is robust to per-query verdicts.

\section*{Ethical Considerations}
\label{sec:ethics}

\paragraph{Synthetic data, no real subjects.}
TGT contains no real-user data: memory text is composed from a
controlled vocabulary of fictional names and slot values
(Appendix~\ref{app:tgt-vocab}) and naturalised by an LLM under a
leakage-prevention protocol (Appendix~\ref{app:tgt-generation}). Biological motivations cite
prior captive-animal behavioural studies; we ran no animal experiments of our own.

\paragraph{Misuse: silent forgetting of sensitive content.}
ScrubJay-MEM selectively forgets information based on an
auto-classified perishability coefficient $\pi_i$. A production
deployment could silently drop user-relevant medical, financial, or
safety-critical content mislabelled as ephemeral. We recommend
treating $\pi_i$ as advisory and gating deletion of any sensitive
memory class behind an explicit retention policy rather than the
classifier alone.

\bibliography{acl-style-files-master/custom}

\appendix
\section{Biological Background: Episodic Memory in Western Scrub Jays}
\label{sec:bio}

Western scrub jays (\textit{Aphelocoma californica}) are among the few
non-human species that exhibit \emph{episodic-like} memory: the ability
to recall the \emph{what}, \emph{where}, and \emph{when} of a past
event as an integrated representation rather than three independent
features~\citep{clayton1998episodic}.  Four empirical findings from
this literature motivate every component of our framework.

\paragraph{Integrated What--Where--When (WWW) retrieval.}
When jays cache perishable (worms) and durable (peanuts) food at
distinct spatial sites, they later recover the correct item from the
correct location after the correct delay, demonstrating that the three
dimensions are bound into a single episodic trace and that cueing any
one dimension activates the others~\citep{clayton1998episodic,
clayton2001integrated}.

\paragraph{Perishability-aware temporal decay.}
Jays preferentially recover worms after short retention intervals (when
worms are still fresh) but switch to peanuts after long intervals (when
worms have decayed).  Memory utility is therefore a joint function of
elapsed time and the item's intrinsic decay rate, not a uniform
recency signal~\citep{clayton2003can}.

\paragraph{Retroactive contextual integration.}
When jays learn \emph{after} caching that a food type decays faster
than expected, they revise their recovery strategy
accordingly~\citep{clayton2001integrated}.  This retroactive update
operates on the existing episodic trace; no new caching episode is
required.

\paragraph{Prospective caching (future planning).}
Jays cache more food in compartments where they anticipate being
hungry the following morning, independently of their current
motivational state~\citep{raby2007planning, correia2007western}.  This
constitutes anticipatory memory loading: encoding is biased by
\emph{future} retrieval needs, not solely past experience.

\noindent
Table~\ref{tab:bio-mapping} summarizes the mapping from each biological
finding to its computational counterpart in ScrubJay-MEM; the
architecture is described in full in \S\ref{sec:method}.

\section{Temporal Generalization Test: Construction Details}
\label{app:tgt}
\begin{table*}[t]
\centering
\caption{Perishability taxonomy (auto-classified at encoding time).}
\label{tab:perishability}
\small
\begin{tabular}{@{}lccp{3.6cm}@{}}
\toprule
\textbf{Memory Type} & $\boldsymbol{\pi}$ & $\boldsymbol{\tau}$ & \textbf{Example} \\
\midrule
Stable knowledge     & 0.05--0.15  & weeks--months & User's occupation, city \\
Procedural / how-to  & 0.2--0.4    & days--weeks   & Coding style preference \\
Task-specific        & 0.5--0.7    & hours--days   & Current project context \\
Ephemeral            & 0.8--1.0    & min--hours    & Today's meeting time \\
\bottomrule
\end{tabular}
\end{table*}
\begin{table*}[t]
\centering
\caption{Mapping from scrub jay cognition to ScrubJay-MEM components.}
\label{tab:bio-mapping}
\small
\begin{tabular}{@{}lll@{}}
\toprule
\textbf{Biological Finding} & \textbf{Cognitive Mechanism} & \textbf{ScrubJay-MEM Component} \\
\midrule
WWW integration        & Bound episodic trace   & Cross-attention EMU (\S\ref{sec:emu}) \\
Perishability          & Item-specific decay    & Utility function $U$ (\S\ref{sec:utility}) \\
Retroactive revision   & Trace parameter update & RCI (\S\ref{sec:rci}) \\
Prospective caching    & Anticipatory loading   & PMB (\S\ref{sec:pmb}) \\
\bottomrule
\end{tabular}
\end{table*}

\begin{figure*}[t]
  \centering
  \includegraphics[width=0.95\linewidth]{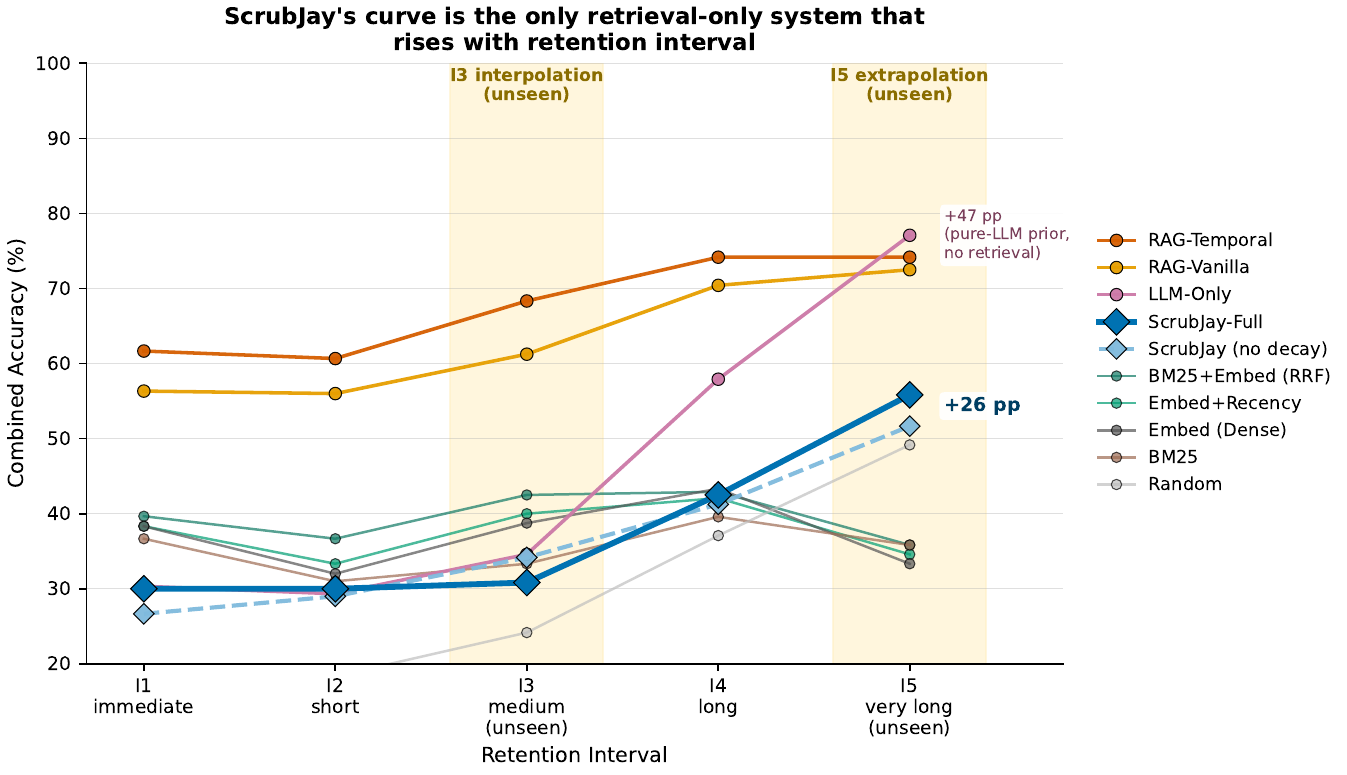}
  \caption{Accuracy across the five retention intervals on E3 TGT.
    Shaded bands mark the two unseen intervals ($I_3$, interpolation;
    $I_5$, extrapolation). ScrubJay-MEM (full) maintains a smooth
    monotonic profile across both unseen intervals; dense, lexical,
    and recency baselines all show characteristic dips at $I_3$
    (negative GenGap). The decay ablation curve shows ScrubJay-MEM
    losing its $I_3$ smoothness when type-conditioned perishability
    is removed.}
  \label{fig:e3-decay}
\end{figure*}

\begin{figure*}[t]
  \centering
  \includegraphics[width=0.78\linewidth]{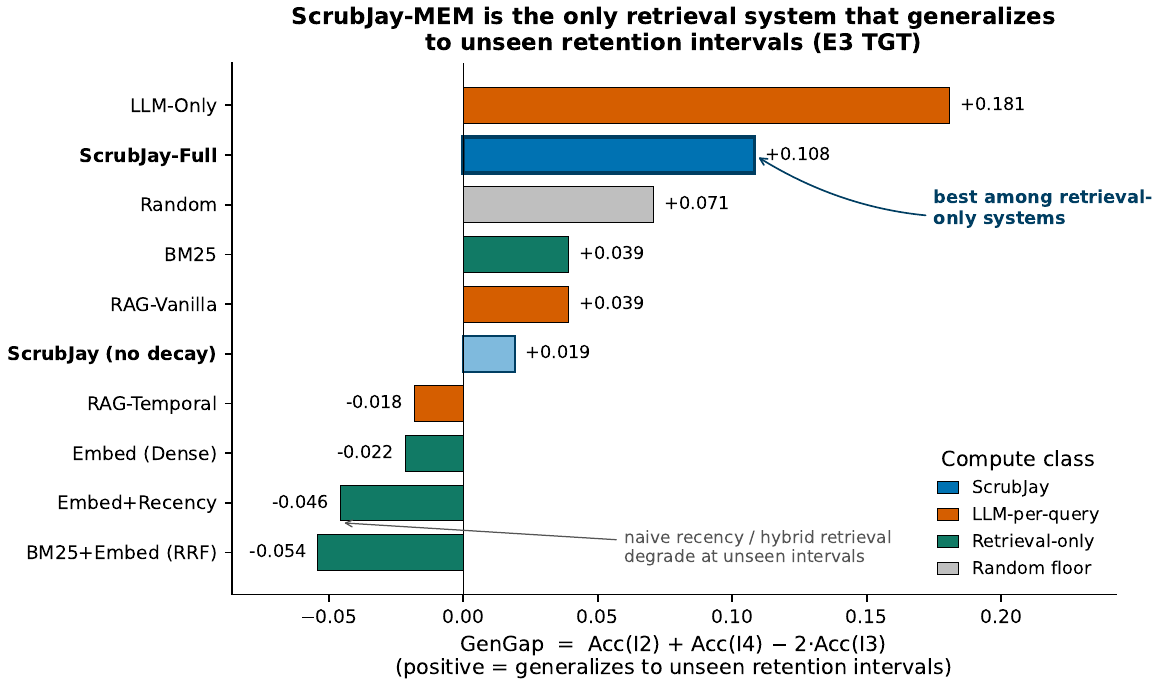}
  \caption{Generalization Gap across all evaluated systems on E3 TGT,
    sorted by GenGap. Among retrieval-based systems (blue), only
    ScrubJay-MEM (full) achieves substantially positive GenGap; every
    flat-retrieval baseline has GenGap $\leq 0$. The decay ablation
    (hatched) reduces GenGap by $5.7\times$, isolating type-conditioned
    perishability as the responsible mechanism. Note that LLM-only
    (B1, gray) has positive GenGap at low factual accuracy ($44.4\%$),
    reflecting that uninformed answers carry no interval-specific
    bias.}
  \label{fig:e3-gengap}
\end{figure*}

\begin{figure*}[t]
  \centering
  \includegraphics[width=0.95\linewidth]{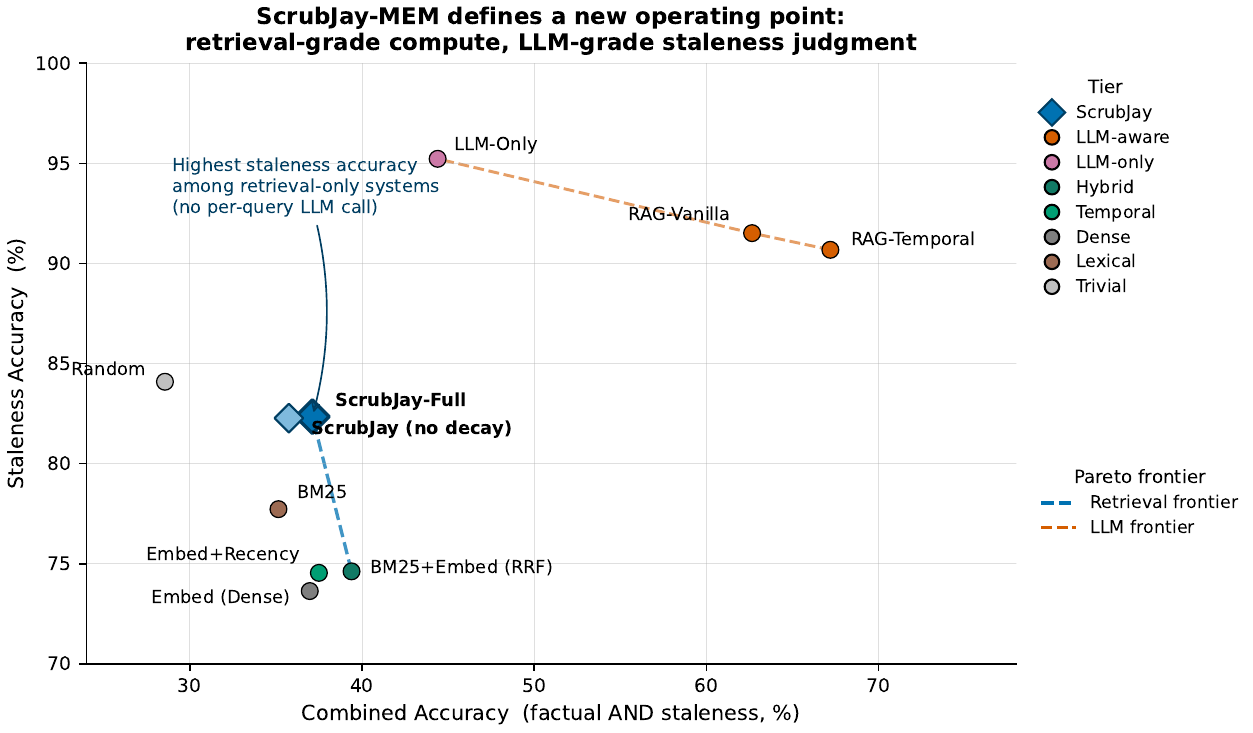}
  \caption{Operating-point view of E3 TGT: staleness accuracy vs.\
    combined (factual $\wedge$ staleness) accuracy. The plot exposes
    two distinct Pareto frontiers: an LLM-frontier traced by
    RAG-vanilla and RAG-temporal (one LLM call per query) and a
    retrieval-frontier traced by lexical/dense/hybrid baselines (no
    per-query LLM call). ScrubJay-MEM sits at the top of the
    retrieval frontier, matching LLM-only's staleness accuracy ($82\%$
    vs.\ $95\%$) at retrieval-grade per-query cost; flat retrieval
    baselines collapse to $73\text{--}78\%$ staleness accuracy. The
    decay ablation collapses GenGap (\Cref{fig:e3-gengap}) but barely
    moves this operating point, confirming that decay drives
    \emph{generalization} rather than mean accuracy.}
  \label{fig:e3-pareto}
\end{figure*}

This appendix specifies the TGT benchmark in full: memory taxonomy,
interval schedule, query design, generation pipeline, controlled
vocabularies, quality audit, and validation methodology. The
benchmark, generation scripts, and a frozen $20$-instance release
artifact is available in the supplementary.

\begin{table*}[t]
  \centering
  \small
  \caption{Ground-truth validity schedule. A memory of type $X$ queried
    at interval $I_k$ is valid iff the corresponding cell is \cmark.
    Held-out intervals are shaded.}
  \label{tab:tgt-validity}
  \begin{tabular}{lccccc}
    \toprule
                       & $I_1$ & $I_2$ & \cellcolor{gray!15}$I_3$ & $I_4$ & \cellcolor{gray!15}$I_5$ \\
                       & immediate & short & \cellcolor{gray!15}medium & long & \cellcolor{gray!15}very long \\
    \midrule
    A: Ephemeral             & \cmark & \cmark & \cellcolor{gray!15}\xmark & \xmark & \cellcolor{gray!15}\xmark \\
    B: Session-Specific      & \cmark & \cmark & \cellcolor{gray!15}\cmark & \xmark & \cellcolor{gray!15}\xmark \\
    C: Durable Preference    & \cmark & \cmark & \cellcolor{gray!15}\cmark & \cmark & \cellcolor{gray!15}\xmark \\
    D: Stable Knowledge      & \cmark & \cmark & \cellcolor{gray!15}\cmark & \cmark & \cellcolor{gray!15}\cmark \\
    \bottomrule
  \end{tabular}
\end{table*}

\subsection{Memory Taxonomy}
\label{app:tgt-types}

Each memory belongs to one of four perishability classes with
type-specific ranges for the ground-truth decay parameters
($\pi_i$, $\tau_i$) used by ScrubJay-MEM and reported alongside every
memory for evaluation purposes:

\begin{table*}[h]
  \centering
  \small
  \caption{Memory taxonomy. $\pi_i$ is the perishability coefficient
    used by ScrubJay-MEM's utility function (\Cref{sec:utility});
    $\tau_i$ is the utility horizon in sessions. Ground-truth values
    are sampled uniformly within the listed ranges. Both parameters
    are stored alongside each memory to enable post-hoc analysis but
    are not visible to systems during retrieval.}
  \label{tab:tgt-types}
  \begin{tabular}{llcccl}
    \toprule
    Type & Label              & $\pi_i$ range & $\tau_i$ (sessions) & Categories & Real-world analog \\
    \midrule
    A & Ephemeral             & 0.80--0.95 & 0.5--2.0    & 5 & current location, today's plan \\
    B & Session-Specific      & 0.55--0.75 & 3.0--8.0    & 5 & recent decisions, active bugs \\
    C & Durable Preference    & 0.15--0.35 & 15.0--40.0  & 5 & tool/food preferences, working style \\
    D & Stable Knowledge      & 0.02--0.10 & 80.0--200.0 & 5 & occupation, family, education \\
    \bottomrule
  \end{tabular}
\end{table*}

\begin{figure*}[t]
  \centering
  \includegraphics[width=\linewidth]{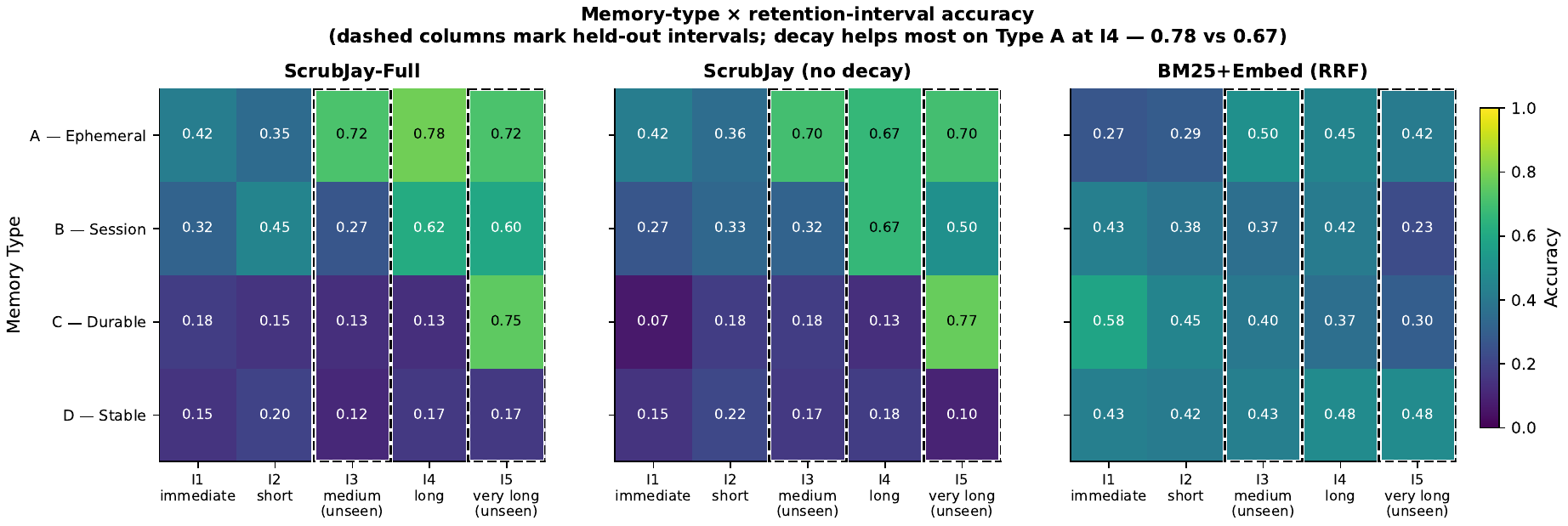}
  \caption{Per-cell accuracy by memory type ($A$--$D$) and retention
    interval ($I_1$--$I_5$). Dashed columns mark the two held-out
    intervals. ScrubJay-MEM (full, left) outperforms the no-decay
    ablation (middle) most strongly on Type-$A$ ephemeral memories at
    long intervals, in particular at $I_4$ ($0.78$ vs.\ $0.67$) and
    $I_5$ ($0.72$ vs.\ $0.70$), where type-conditioned perishability
    correctly suppresses stale ephemeral facts. The flat-hybrid
    baseline (BM25+Embed RRF, right) wins on the Type-$D$ stable row
    by retrieving everything indiscriminately; ScrubJay-MEM trades
    Type-$D$ accuracy for stronger generalization across the
    perishable types.}
  \label{fig:tgt-typeinterval}
\end{figure*}

Twenty categories total span the four classes:
\begin{description}\itemsep0pt
  \item[Type A:] \texttt{current\_location}, \texttt{todays\_plan},
        \texttt{transient\_mood}, \texttt{immediate\_context},
        \texttt{real\_time\_status}.
  \item[Type B:] \texttt{project\_decision}, \texttt{meeting\_outcome},
        \texttt{recent\_preference\_shift}, \texttt{active\_bug},
        \texttt{short\_term\_goal}.
  \item[Type C:] \texttt{food\_preference}, \texttt{communication\_style},
        \texttt{tool\_preference}, \texttt{working\_hours},
        \texttt{aesthetic\_preference}.
  \item[Type D:] \texttt{occupation}, \texttt{location}, \texttt{long\_term\_goal},
        \texttt{family\_info}, \texttt{education}.
\end{description}

\subsection{Retention Intervals}
\label{app:tgt-intervals}

Each instance issues queries at five retention intervals. The session
offset $\Delta t$ for each query is sampled uniformly within the
interval's session range:

\begin{table*}[h]
  \centering
  \small
  \caption{Retention interval schedule. ``Held out'' intervals are
    never available for calibration; they test the system's ability
    to interpolate ($I_3$) or extrapolate ($I_5$) its temporal model.}
  \label{tab:tgt-intervals}
  \begin{tabular}{lllcl}
    \toprule
    Interval & Label & Session range & Held out & Avian analog \\
    \midrule
    $I_1$ & Immediate  & 0--1   & no   & 4h (short, trained) \\
    $I_2$ & Short      & 3--5   & no   & known short delay \\
    $I_3$ & Medium     & 8--12  & \textbf{yes} & \textbf{novel intermediate delay} \\
    $I_4$ & Long       & 20--25 & no   & 28h (long, trained) \\
    $I_5$ & Very Long  & 40--60 & \textbf{yes} & \textbf{novel long delay (extrapolation)} \\
    \bottomrule
  \end{tabular}
\end{table*}

\subsection{Query Types}
\label{app:tgt-queries}

For every (interval, type) pair, three queries are generated, totaling
$5 \times 4 \times 3 = 60$ \emph{base} queries per instance:

\begin{enumerate}\itemsep0pt
  \item \textbf{Factual recall:} \emph{``What is $\{$subject$\}$'s
        $\{$attribute\_label$\}$?''} — tests retrieval grounded in the
        target memory.
  \item \textbf{Comparative staleness:} \emph{``Has $\{$subject$\}$'s
        $\{$attribute\_label$\}$ changed recently?''} — tests temporal
        validity judgment. By construction, this template never
        contains the answer (see audit, \Cref{app:tgt-audit}).
  \item \textbf{Prospective:} \emph{``What should I remember about
        $\{$subject$\}$'s $\{$attribute\_label$\}$ for
        $\{$future\_task$\}$?''} — tests forward-looking memory use.
\end{enumerate}

\subsection{Stale Distractor Queries}
\label{app:tgt-stale}

To prevent the system from defaulting to ``always valid'' at early
intervals, $6$ \emph{stale-distractor} queries are injected at $I_1$
and $I_2$ (one per query type per interval, $2 \times 3 = 6$). Each
targets a Type-A memory explicitly marked
$\mathrm{ground\_truth\_valid}=\mathrm{False}$ to simulate an
already-expired ephemeral fact (e.g., the user's location ``today'' is
no longer current because the session is one day later). With
distractors, the always-valid trivial baseline drops from $70\%$ to
$63.6\%$, restoring discriminative signal at the two earliest intervals.

Each instance thus contains $60 + 6 = 66$ queries.

\subsection{Generation Pipeline}
\label{app:tgt-generation}

TGT supports two generation modes. \emph{Template mode} composes
memories from controlled vocabularies (see \Cref{app:tgt-vocab})
using deterministic templates for each category; it is used for unit
testing and reproducibility checks. \emph{Ollama mode} adds an LLM
naturalization pass (\texttt{qwen3:30b-instruct}) that rewrites both
memories and queries in natural prose while preserving the factual
slots. All reported numbers are from Ollama mode.

\paragraph{Memory naturalization.}
The LLM is instructed to rewrite the template memory keeping every
factual slot (subject, attribute, value, future task) intact and
returning a single declarative sentence of at most two clauses. A
checkpoint tag of the form ``\texttt{This was logged in checkpoint
P\{pool\_id\}-\{type\_id\}\{index\}.}'' is appended to every memory
to give the retrieval layer a unique identifier-like surface form.

\paragraph{Query naturalization with leakage prevention.}
For \emph{comparative staleness} queries the LLM prompt includes
\texttt{FORBIDDEN ANSWER CONTENT: \{answer\_text\}} and instructs
the model to ask about the attribute generically without
revealing the value. A post-hoc safety check rejects any
naturalized query whose lowercased text contains
\texttt{answer\_text}; the template is used as a fallback.

\paragraph{Diversity controls.}
\textbf{Without-replacement targeting:} within an instance, each
(interval, type) pair selects a target memory not yet used by any
other interval for that type; the used-set resets only when the
type's pool of $24$ memories is exhausted. This raises intra-instance
query-text uniqueness from $\sim$65\% (with replacement) to a measured
mean of $96\%$ (min $88\%$, max $100\%$).
\textbf{One pool per instance:} the production configuration uses
$20$ pools for $20$ instances ($1{:}1$). Earlier configurations sharing
$5$ pools across $20$ instances produced only $5$ statistically
independent samples and were rejected by the audit (\Cref{app:tgt-audit}).

\subsection{Controlled Vocabularies}
\label{app:tgt-vocab}

Memory and query templates draw from the following pools:
$12$ names, $6$ office areas, $6$ cities, $6$ project codenames,
$5$ filenames, $5$ programming languages, $5$ developer tools,
$5$ data stores, $5$ frameworks, $5$ foods, $4$ communication styles,
$4$ aesthetic preferences, $6$ companies, $5$ universities,
$6$ child names, $4$ long-term goals, $4$ build statuses, $4$ mood
descriptors. Together these support roughly
$5 \cdot 10^{11}$ distinct memory slot combinations before LLM
naturalization, far exceeding the $1{,}920$ memories used in the full
benchmark and ensuring the validation step (\Cref{app:tgt-audit})
finds no duplicate memory text.

\subsection{Metrics}
\label{app:tgt-metrics}

Per-query scoring is performed by an LLM judge (\texttt{llama3.2:3b},
from a different model family than the generation backbone
\texttt{qwen3:30b-instruct}) which returns a structured JSON with two
binary fields:

\begin{itemize}\itemsep0pt
  \item \texttt{factual\_correct} $\in \{0, 1\}$ — does the answer
        contain the ground-truth attribute value?
  \item \texttt{staleness\_aware} $\in \{0, 1\}$ — is the answer's
        validity judgment consistent with the ground-truth
        type$\times$interval schedule (\Cref{tab:tgt-validity})?
\end{itemize}

\emph{Combined accuracy} requires both: $\mathrm{combined}(q) =
\mathbf{1}[\texttt{factual\_correct}(q) = 1 \wedge \texttt{staleness\_aware}(q) = 1]$.
Per-interval accuracy is the mean over queries with $q.\mathrm{interval} = I_k$.
The Temporal Generalization Score is the mean across all five intervals,
$\mathrm{TGS}(f) = \tfrac{1}{5} \sum_{k=1}^{5} \mathrm{Acc}(f, I_k)$.
GenGap is given by \Cref{eq:gengap}. Additional decompositions
reported in the appendix include factual and staleness accuracy in
isolation, per-type-per-interval accuracy, Recall@$5$, MRR, and
per-query latency.

A heuristic fallback judge (token-overlap factual matching with a
60\% threshold; staleness detected via uncertainty markers such as
``stale'', ``outdated'', ``cannot confirm'') is available for
reproducibility runs without an LLM judge, but is not used for any
reported numbers.

\begin{figure}[t]
  \centering
  \includegraphics[width=\linewidth]{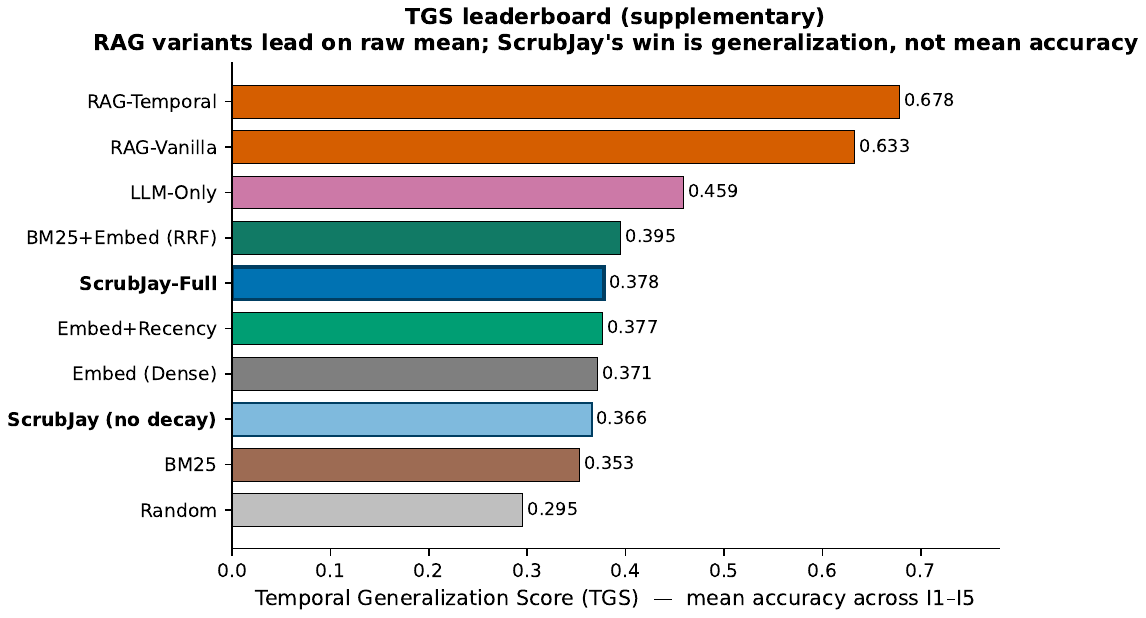}
  \caption{Temporal Generalization Score (TGS) leaderboard: mean
    accuracy across $I_1$--$I_5$. RAG variants lead on raw mean
    accuracy because of LLM-driven answer generation; the
    retrieval-only band (BM25, dense, hybrid, ScrubJay variants) is
    tightly clustered. ScrubJay-MEM's contribution is on the
    generalization axis (\Cref{fig:e3-gengap}), not on raw TGS.}
  \label{fig:tgt-tgs-leaderboard}
\end{figure}

\subsection{Decay Ablation: Per-Metric Decomposition}
\label{app:tgt-ablation}

\Cref{fig:e3-decay} visualizes the decay ablation across retention
intervals; \Cref{fig:tgt-ablation} decomposes the same ablation
across the four TGT metrics. Removing type-conditioned decay leaves
factual, staleness, and combined accuracy essentially unchanged
($|\Delta|\!\le\!1.3$ percentage points each), but collapses GenGap
from $+10.8$ to $+1.9$ (a $5.7\times$ reduction). This decomposition
is the basis for our claim that decay is the mechanism responsible
for generalization rather than for mean accuracy.

\begin{figure}[t]
  \centering
  \includegraphics[width=\linewidth]{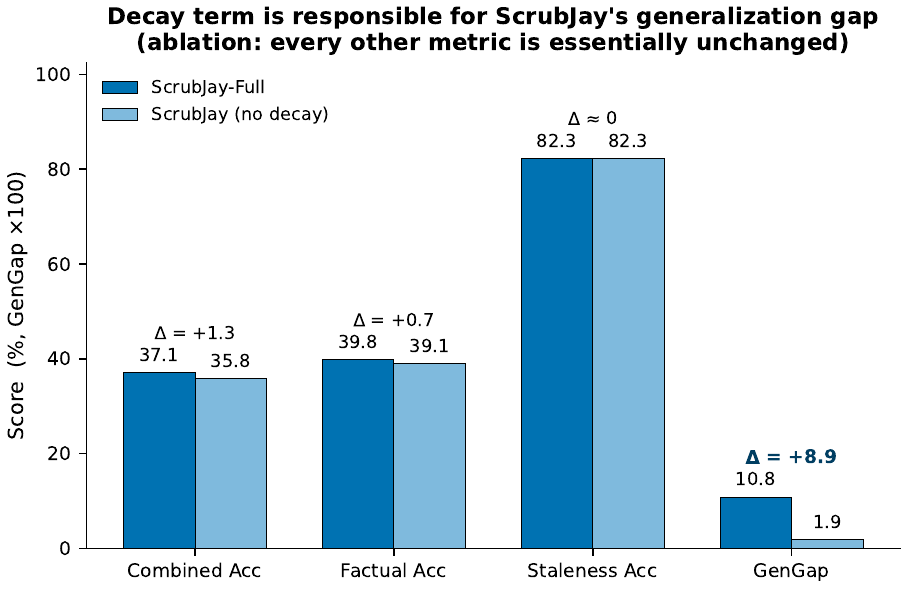}
  \caption{Decay ablation decomposed per metric. ScrubJay-MEM (full)
    vs.\ ScrubJay-MEM (no decay) on Combined, Factual, and Staleness
    accuracy (left three groups) and GenGap (right, $\times 100$ for
    visual scale). All non-GenGap metrics shift by $\leq 1.3$ points;
    GenGap moves by $+8.9$ points. Type-conditioned decay is isolated
    as the generalization mechanism.}
  \label{fig:tgt-ablation}
\end{figure}

\subsection{Trivial Baselines}
\label{app:tgt-trivial}

Several trivial strategies establish the floor of the metric:

\begin{table*}[h]
  \centering
  \small
  \caption{Trivial baseline combined accuracy on TGT.}
  \label{tab:tgt-trivial}
  \begin{tabular}{ll c}
    \toprule
    Strategy & Description & Combined Acc \\
    \midrule
    Always-valid & Treat every memory as still valid                   & 63.6\% \\
    Always-stale & Treat every memory as outdated                      & 36.4\% \\
    Per-interval majority & Predict majority validity per interval     & 70\% (mean) \\
    \bottomrule
  \end{tabular}
\end{table*}

The always-valid floor of $63.6\%$ is a consequence of the validity schedule (\Cref{tab:tgt-validity}) and the stale-distractor injection;
no system should be credited for combined accuracy below this baseline
on factual+staleness joint scoring.

\subsection{Quality Audit and Construction Fixes}
\label{app:tgt-audit}

The initial benchmark generation surfaced several issues, each of which
was resolved before the reported runs. We report the audit transparently
because two of the issues (answer leakage and pool reuse) would have
materially inflated baseline scores if uncorrected.

\begin{table*}[h]
  \centering
  \small
  \caption{Quality audit: pre- and post-fix metrics. All numbers
    measured against $20$ generated instances ($1{,}320$ queries).
    The reported experiments use only the post-fix benchmark.}
  \label{tab:tgt-audit}
  \begin{tabular}{lccl}
    \toprule
    Criterion                       & Before fixes & After fixes & Threshold \\
    \midrule
    Answer leakage rate             & 8.7\% (104/1200) & \textbf{0.0\%} & 0\% required \\
    Effective independent samples   & 5            & \textbf{20}    & $\geq 20$ preferred \\
    Memory uniqueness               & 25\%         & \textbf{100\%} & $>$95\% \\
    Intra-instance query uniqueness & $\sim$65\% (min) & \textbf{96\% mean} & $>$90\% \\
    Stale queries at $I_1$          & 0\%          & \textbf{20\%}  & $>$0\% required \\
    Stale queries at $I_2$          & 0\%          & \textbf{20\%}  & $>$0\% required \\
    Always-valid baseline           & 70.0\%       & \textbf{63.6\%} & $<$70\% preferred \\
    Per-cell sample size (min)      & 20           & \textbf{20}    & $\geq 10$ required \\
    Keyword solvability             & 8.0\%        & \textbf{0.0\%} & $<$5\% \\
    Validation errors               & 0            & \textbf{0}     & 0 required \\
    \bottomrule
  \end{tabular}
\end{table*}

The five concrete fixes were:
\textbf{(F1) Answer leakage} eliminated by rewriting the
\emph{comparative staleness} template to ask about the attribute
generically (``Has $X$'s $Y$ changed recently?''), never about the
value, and by adding the \texttt{FORBIDDEN ANSWER CONTENT} clause and
post-hoc safety check to query naturalization.
\textbf{(F2) Pool reuse} eliminated by setting
\texttt{e3\_num\_pools $=$ e3\_instances} so each instance gets a
unique memory pool.
\textbf{(F3) Intra-instance duplication} reduced by adding
without-replacement target selection per (interval, type) pair.
\textbf{(F4) Missing staleness at early intervals} resolved by
injecting six stale-distractor queries per instance at $I_1$ and $I_2$.
\textbf{(F5) Validation strengthened} to include leakage detection,
intra-instance uniqueness, pool-reuse ratios, and stale-query
presence as audit conditions.

\subsection{Validation Methodology}
\label{app:tgt-validation}

Programmatic validation runs on every generated dataset before the
benchmark is released to a run. The validator checks: memory counts
per type (must equal $24$ each); memory uniqueness within instance
(no duplicate \texttt{text}); target memory presence in candidate
pool; consistency between $\mathrm{ground\_truth\_valid}$ and the
type$\times$interval schedule (stale distractors exempted); absence
of \texttt{answer\_text} verbatim in \texttt{query\_text}; per-cell
sample size; and pool-reuse ratio.

The released benchmark passes all checks with zero errors.

\subsection{Comparison to Existing Benchmarks}
\label{app:tgt-comparison}

\begin{table*}[h]
  \centering
  \small
  \setlength{\tabcolsep}{4pt}
  \renewcommand{\arraystretch}{0.95}
  \caption{TGT vs.\ public memory benchmarks. ``Held-out intervals''
    refers to retention intervals never seen during any calibration
    phase; only TGT explicitly tests interpolation/extrapolation of
    the temporal model.}
  \label{tab:tgt-compare}
  \begin{tabular}{lcccc}
    \toprule
    & LongMemEval & LoCoMo & MemoryAgentBench & TGT (ours) \\
    & \citep{wu2025longmemeval}
    & \citep{maharana2024locomo}
    & \citep{hu2025memoryagentbench}
    & \\
    \midrule
    Total queries
      & 500 & $\sim$2{,}400 & varies & 1{,}320 \\
    Effective indep.\ samples
      & 500 & 50 & varies & 20 \\
    Temporal dimension
      & timestamps & sessions & sessions
      & \textbf{\begin{tabular}[c]{@{}c@{}}continuous +\\ held-out intervals\end{tabular}} \\
    Held-out intervals
      & no & no & no
      & \textbf{yes ($I_3$, $I_5$)} \\
    Staleness at every interval
      & partial & no & partial & \textbf{yes} \\
    Perishability classes
      & n/a & n/a & n/a
      & \textbf{4 ($\pi$-stratified)} \\
    GenGap metric
      & n/a & n/a & n/a & \textbf{this work} \\
    \bottomrule
  \end{tabular}
\end{table*}
TGT is complementary to existing benchmarks rather than a replacement:
LongMemEval and LoCoMo test long-horizon recall over realistic
conversation histories; MemoryAgentBench tests several skills including
fact retrieval and conflict resolution. TGT isolates the temporal
generalization axis specifically, with held-out intervals and a
type-stratified validity schedule. We report ScrubJay-MEM on
MemoryAgentBench EventQA (\Cref{sec:mab-eventqa}) for direct comparison
to published memory systems, and on TGT (\Cref{sec:e3-tgt}) for the
mechanistic claim.

\subsection{Release}
\label{app:tgt-release}

We provide in supplementary:
(i) the frozen $20$-instance benchmark used in
\Cref{sec:e3-tgt} ($1{,}320$ queries, $1{,}920$ memories) as a single
JSON file with full ground-truth annotations;
(ii) the generation scripts and seeds needed to reproduce the
benchmark, including the controlled vocabularies and template
specifications;
(iii) the LLM-judge prompts and the heuristic fallback judge;
(iv) reference implementations of all nine baselines evaluated in
\Cref{sec:e3-tgt}.


\section{EventQA-64k with stronger backbone (\texttt{qwen3:30b})}
\label{app:eventqa-qwen}

To probe whether ScrubJay-MEM's advantage in \Cref{sec:mab-eventqa}
depends on backbone scale, we re-ran the EventQA-64k sweep with
\texttt{qwen3:30b-instruct} as the generation model while keeping the
embedding model (\texttt{nomic-embed-text}) and retrieval depth ($k=5$)
fixed. With this larger backbone, dense retrieval closes the gap and
slightly overtakes ScrubJay-MEM, suggesting that some of the temporal
inductive bias the architecture provides can be recovered at inference
time by a capable LLM reasoning over dense retrievals
(\Cref{tab:eventqa-qwen30b}).

\begin{table}[h]
  \centering
  \small
  \caption{EventQA-64k with \texttt{qwen3:30b-instruct} backbone.
    With a more capable backbone, dense-retrieval baselines close the
    gap, indicating that ScrubJay-MEM's largest gains are in the
    small-LLM regime.}
  \label{tab:eventqa-qwen30b}
  \begin{tabular}{lcc}
    \toprule
    System                  & EM    & F1   \\
    \midrule
    Mem0                    & 56.60 & 72.78 \\
    ScrubJay-MEM (ours)     & 58.20 & 72.82 \\
    BM25                    & 65.00 & 78.02 \\
    Qwen3-Embedding-4B      & 65.80 & 78.53 \\
    Contriever              & 66.80 & 78.68 \\
    \bottomrule
  \end{tabular}
\end{table}

\section{Conflict-Resolution: where ScrubJay-MEM does not help}
\label{app:mab-cr}

\Cref{sec:mab-eventqa} flags that ScrubJay-MEM is miscalibrated on
MemoryAgentBench Conflict-Resolution: type-conditioned decay suppresses
the very memories the consolidation step requires.
\Cref{tab:mab-cr-appendix} reports the numbers.

\begin{table}[h]
  \centering
  \small
  \caption{MemoryAgentBench Conflict-Resolution subsets
    (Factconsolidation-MH, Factconsolidation-SH).
    ScrubJay-MEM underperforms dense and lexical baselines by 14--18
    F1: type-conditioned decay suppresses facts that the
    consolidation task requires.}
  \label{tab:mab-cr-appendix}
  \begin{tabular}{lcc}
    \toprule
    System              & CR-MH F1 & CR-SH F1 \\
    \midrule
    Mem0                & 0.22     & 1.66     \\
    ScrubJay-MEM (ours) & 0.36     & 7.63     \\
    BM25                & 1.00     & 21.48    \\
    Contriever          & 5.40     & 22.40    \\
    Qwen3-Embedding-4B  & n/a      & 20.59    \\
    \bottomrule
  \end{tabular}
\end{table}

\section{Proof of Proposition~\ref{prop:rci}}
\label{app:proofs}

Following \S\ref{sec:rci}, the per-event LLM call emits bounded
deltas $(\Delta V(\mathbf{n}),\,\Delta\pi(\mathbf{n}),\,
\Delta\tau(\mathbf{n}))$, with the dynamic ranges enforced by parsing
the JSON response and clipping out-of-range outputs. The
threshold-rescaled affinity satisfies $\alpha_{in} \in [0, 1]$ by
construction. The four claims of Proposition~\ref{prop:rci} follow
componentwise.

\paragraph{(i) Bounded $V$-update.}
The update is the linear map
\[
  T_V(V_i) \;=\; \max\!\bigl(0,\; V_i + \eta_V\,\alpha_{in}\,\Delta V(\mathbf{n})\bigr).
\]
The change in $V_i$ satisfies
$|T_V(V_i) - V_i| \le \eta_V\,\alpha_{in}\,|\Delta V(\mathbf{n})|
\le \eta_V\,\Delta V_{\max}$, where the inequality uses
$\alpha_{in} \le 1$. The non-negativity clip is a non-expansion onto
$\mathbb{R}_{\ge 0}$ and preserves the bound.

\paragraph{(ii) Bounded $\pi$-update.}
Without the clip, the same argument applies to
$\pi_i + \eta_\pi\,\alpha_{in}\,\Delta\pi(\mathbf{n})$, giving
$|\Delta\pi_i| \le \eta_\pi\,\Delta\pi_{\max}$. The clip
$\operatorname{clip}(\cdot, 0, 1)$ is the projection of $\mathbb{R}$
onto $[0, 1]$ and is $1$-Lipschitz (non-expansive), so the post-clip
change is no larger than the pre-clip change. Hence
$|\Delta\pi_i|_{\text{post-clip}} \le \eta_\pi\,\Delta\pi_{\max}$
inside $[0, 1]$.

\paragraph{(iii) $\tau$ confined to $[\tau_{\min}, \tau_{\max}]$.}
The update is
\begin{equation}
\small
\begin{aligned}
  T_\tau(\tau_i)
  &=
  \operatorname{clip}\!\Bigl(
    \tau_i \cdot
    \max\!\bigl(
      0.1,\,
      1 + \eta_\tau\,\alpha_{in}\,\Delta\tau(\mathbf{n})
    \bigr),
    \\
  &\hspace{2.8cm}
    \tau_{\min},\,\tau_{\max}
  \Bigr).
\end{aligned}
\label{eq:tau-transform}
\end{equation}
The inner $\max(0.1,\cdot)$ guarantees the multiplicative factor is
strictly positive (so $\tau$ never becomes non-positive), and the
outer clip projects onto the operating interval. Hence
$T_\tau : [\tau_{\min}, \tau_{\max}] \to [\tau_{\min}, \tau_{\max}]$
for any choice of $\alpha_{in}$ and $\Delta\tau$ in the assumed
ranges. The multiplicative form does not in general contract toward a
unique target. Under repeated application of the same
$\Delta\tau\!>\!0$ it saturates against $\tau_{\max}$, and against
$\tau_{\min}$ for repeated $\Delta\tau\!<\!0$. The clip endpoints are
absorbing.

\paragraph{(iv) $\mathbf{w}^{\textsc{what}}$-update is a contraction toward $\mathbf{w}_n$.}
Ignoring the L2 normalisation step (which is a radial projection onto
the unit sphere and is non-expansive in the chordal metric), the
update~\eqref{eq:rci-embed} can be rewritten as
\[
  T_w(\mathbf{w}_i^{\textsc{what}})
  \;=\;
  (1 - \eta_w\,\alpha_{in})\,\mathbf{w}_i^{\textsc{what}}
  + \eta_w\,\alpha_{in}\,\mathbf{w}_n^{\textsc{what}}.
\]
For any $\mathbf{u}, \mathbf{v} \in \mathbb{R}^{d_1}$,
\[
  \|T_w(\mathbf{u}) - T_w(\mathbf{v})\|
  \;=\; (1 - \eta_w\,\alpha_{in})\,\|\mathbf{u} - \mathbf{v}\|,
\]
so $T_w$ is a contraction in the Euclidean norm with Lipschitz
constant $(1 - \eta_w\,\alpha_{in}) \in [0, 1)$ whenever
$\eta_w\,\alpha_{in} < 1$. The Banach fixed-point theorem then gives
a unique fixed point $\mathbf{w}_n^{\textsc{what}}$ and geometric
convergence at rate $(1 - \eta_w\,\alpha_{in})^k$ when
$\mathbf{w}_n^{\textsc{what}}$ is held fixed across iterations. The
final L2 normalisation is a non-expansion on the unit sphere and
preserves the contraction property in the chordal metric. $\square$

\paragraph{Practical reading.}
Claims (i)-(iii) say that no single RCI event can move a memory's
latent parameters by more than the LLM's parsed delta times the
learning-rate-affinity product, and that the horizon $\tau$ remains
inside its operating window indefinitely. Claim (iv) says that the
\textsc{what} embedding moves \emph{toward} the new evidence at a
controlled rate; this is the only update among the four with a
genuine contraction structure under the implementation. The
additive forms in (i)-(ii) and the multiplicative form in (iii)
guarantee boundedness rather than convergence to a fixed point;
unbounded drift under repeated application is prevented by the
$\alpha_{in}$ gate (irrelevant evidence has $\alpha_{in}\!=\!0$) and
the clips, not by a contraction property.

\section{Hyperparameters, Prompts, and Compute Environment}
\label{app:hyperparams}

This appendix specifies every numerical and textual constant used by
the v1 reference implementation of ScrubJay-MEM, so that an
independent replication can produce the numbers reported in
\S\ref{sec:experiments} from the released code without further tuning.

\subsection{Hyperparameters}
\label{app:hyperparams-values}

Engine defaults (\texttt{EngineConfig} in the reference implementation):

\begin{table*}[h]
  \centering
  \small
  \caption{ScrubJay-MEM v1 hyperparameter values.}
  \label{tab:hyperparams}
  \begin{tabular}{lll}
    \toprule
    Symbol & Value & Role \\
    \midrule
    $\eta_V$              & $0.20$            & RCI learning rate for $V_i$ \\
    $\eta_\pi$            & $0.15$            & RCI learning rate for $\pi_i$ \\
    $\eta_\tau$           & $0.15$            & RCI learning rate for $\tau_i$ (multiplicative) \\
    $\eta_w$              & $0.20$            & RCI learning rate for $\mathbf{w}^{\textsc{what}}$ \\
    $\theta_{\text{RCI}}$ & $0.60$            & cosine gate for affected-memory selection \\
    $\theta_{\mathcal{B}}$ (\texttt{theta\_buffer}) & $0.45$ & PMB confidence threshold \\
    $\theta_S$, $\theta_C$ & $0.35$, $0.35$  & semantic / contextual edge thresholds \\
    $\epsilon$, $\epsilon_G$ & $0.05$, $0.05$ & utility / graph-connectivity prune thresholds \\
    $w_S$, $w_C$, $w_T$    & $0.40$, $0.30$, $0.30$ & graph mixture (\S\ref{sec:graph}) \\
    $\lambda_T$            & $1{/}86400~\mathrm{sec}^{-1}$ & temporal-edge decay rate \\
    $\Delta_T$             & $7$ days          & temporal-edge window \\
    $\tau_{\min}$, $\tau_{\max}$ & $60$ s, $90$ days & utility-horizon clamp \\
    candidate\_k           & $128$             & first-stage candidate pool \\
    default top-$k$        & $32$              & retrieval depth (engine default) \\
    EventQA top-$k$        & $5$               & per \S\ref{sec:mab-eventqa} \\
    TGT top-$k$            & $10$              & per \S\ref{sec:e3-tgt} \\
    embedding\_dim         & $768$             & \texttt{nomic-embed-text} native \\
    seeds                  & $\{42\}$          & single seed used for reported runs \\
    \bottomrule
  \end{tabular}
\end{table*}

\subsection{Compute environment and benchmark-specific model choices}
\label{app:hyperparams-compute}

All reported numbers are produced by a single-machine local deployment: a workstation with an NVIDIA RTX A4500 20GB GPU running an Ollama daemon at localhost:11434 that serves the chat and judge models on demand and the embedding model (nomic-embed-text, 768-dim).
The reference implementation is single-process Python; FAISS is used
as the dense-index backend when available, with a numpy fallback.
The two benchmarks use different model configurations
(\Cref{tab:benchmark-models}); the EventQA configuration uses
MemoryAgentBench's official non-LLM scorer, while TGT uses an LLM
judge from a different model family than the system generator to
mitigate self-preference bias~\citep{panickssery2024llmselfpref}.

\begin{table*}[h]
  \centering
  \small
  \caption{Benchmark-specific model configurations for the runs
    reported in this paper. Family abbreviations: Meta (M), Alibaba
    (A), Nomic (N).}
  \label{tab:benchmark-models}
  \begin{tabular}{lll}
    \toprule
    Component & EventQA-64k (\S\ref{sec:mab-eventqa}) & TGT (\S\ref{sec:e3-tgt}) \\
    \midrule
    System generation backbone & \texttt{llama3.1:8b} (M) & \texttt{qwen3:30b-instruct} (A) \\
    Embedding model            & \texttt{nomic-embed-text} (N) & \texttt{nomic-embed-text} (N) \\
    Benchmark naturalisation   & n/a                       & \texttt{qwen3:30b-instruct} (A) \\
    Answer-correctness judge   & MAB official scorer (no LLM) & \texttt{llama3.2:3b} (M) \\
    Retrieval depth ($k$)      & $5$                       & $10$ \\
    Chunk size                 & $4096$ tokens             & n/a \\
    Random seed                & $42$                      & $42$ \\
    \bottomrule
  \end{tabular}
\end{table*}

We report no wall-clock or throughput numbers because absolute
timings depend on the Ollama configuration of the host machine, but
component-level latencies (LLM call per RCI event, embedding latency
per ingest) are recorded in the released run artefacts.

\subsection{Perishability-classification prompt}
\label{app:hyperparams-prompts-pi}

The LLM-prompt classifier $\phi$ of \S\ref{sec:utility} uses the
following template; the response is required to be strict JSON.

\begin{quote}\small\ttfamily
Classify memory perishability and utility horizon. Return strict
JSON: \{"label": str, "pi": float, "tau\_sec": float\}.\\
Labels: factual, procedural, task\_specific, ephemeral.\\
Context: \{context\}\\
Memory: \{text\}
\end{quote}

On parse failure or LLM unavailability, the system falls back to a
deterministic keyword heuristic over the concatenated memory and
context text: presence of any of \{\texttt{today},
\texttt{immediate}, \texttt{right now}, \texttt{session},
\texttt{temporary}\} $\to$ ephemeral with $\pi\!=\!0.9$,
$\tau\!=\!2$\,h; \{\texttt{how to}, \texttt{steps},
\texttt{process}, \texttt{procedure}, \texttt{workflow}\} $\to$
procedural with $\pi\!=\!0.3$, $\tau\!=\!10$ days;
\{\texttt{task}, \texttt{ticket}, \texttt{issue},
\texttt{meeting}, \texttt{project}\} $\to$ task-specific with
$\pi\!=\!0.6$, $\tau\!=\!24$\,h. Otherwise the memory is treated as
factual / world knowledge with $\pi\!\approx\!0.1$ and $\tau$ set to
$45$ days, with a small ($\le 0.05$) contextual-divergence
perturbation.

\subsection{RCI delta-extraction prompt}
\label{app:hyperparams-prompts-rci}

The RCI $\Delta$-parser of \S\ref{sec:rci} uses:

\begin{quote}\small\ttfamily
Given new information, infer latent memory update deltas. Return
strict JSON: \{"delta\_value": float, "delta\_pi": float,
"delta\_tau": float, "summary": str\}.\\
delta\_tau is relative ratio change where +0.2 means +20\%.\\
Context: \{context\}\\
New information: \{new\_info\_text\}
\end{quote}

The three numerical deltas are bounded at parsing time before being
applied via Eqs.~\eqref{eq:rci-v}--\eqref{eq:rci-tau}.

\subsection{Future-oriented annotation prompt}
\label{app:hyperparams-prompts-annot}

The prospective annotation of \S\ref{sec:encoding} concatenates the
LLM's response to the following prompt onto the raw memory text
before embedding:

\begin{quote}\small\ttfamily
Summarize in one sentence when this memory will be useful in future
tasks. Return plain text only.\\
Context: \{context\}\\
Memory: \{text\}
\end{quote}

\end{document}